\pdfoutput=1
\documentclass[11pt]{article}

\usepackage[letterpaper,margin=1in]{geometry}
\usepackage[numbers,sort&compress]{natbib}
\usepackage{lmodern}
\usepackage[utf8]{inputenc}
\usepackage[T1]{fontenc}
\usepackage{hyperref}
\usepackage{url}
\usepackage{booktabs}
\usepackage{longtable}
\usepackage{amsfonts}
\usepackage{nicefrac}
\usepackage{microtype}
\usepackage{xcolor}
\usepackage{amsmath,amssymb,amsthm,mathtools,bm}
\usepackage{graphicx}
\usepackage{wrapfig}

\hypersetup{colorlinks=true,linkcolor=blue!50!black,citecolor=blue!50!black,urlcolor=blue!50!black}
\theoremstyle{plain}
\newtheorem{theorem}{Theorem}[section]
\newtheorem{proposition}[theorem]{Proposition}

\newtheorem{corollary}[theorem]{Corollary}

\theoremstyle{definition}

\theoremstyle{remark}

\newcommand{\R}{\mathbb{R}}
\newcommand{\E}{\mathbb{E}}

\newcommand{\tr}{\operatorname{tr}}
\newcommand{\rank}{\operatorname{rank}}
\newcommand{\diag}{\operatorname{diag}}
\newcommand{\vecop}{\operatorname{vec}}
\newcommand{\Cov}{\operatorname{Cov}}
\newcommand{\Var}{\operatorname{Var}}
\newcommand{\msign}{\operatorname{msgn}}

\newcommand{\op}{\mathrm{op}}

\title{Muon Under Gradient Noise and the Limits of Orthogonalization Near Optima}

\author{%
  Xiaohui Xie\\[2pt]
  Department of Computer Science\\
  University of California, Irvine\\
  \texttt{xhx@uci.edu}
}
\date{}

\begin{document}
\maketitle
\begin{abstract}
Muon replaces the momentum buffer of each weight matrix by its orthogonal
polar factor. We ask what this orthogonalization does near an optimum, where
minibatch noise dominates the gradient. Under a Gaussian noise model, the
expected Muon update becomes a scaled gradient step, so a linear method with a
suitably matched learning rate reproduces Muon's first-order mean response.
The stochastic update is a different matter: after the response is matched,
Muon retains a nonlinear residual that is uncorrelated with the input noise
and contributes additional covariance. A Hermite expansion shows how momentum
acts on this residual. Its higher-order components decorrelate faster than
the linear component, so momentum suppresses the residual's accumulated
covariance relative to the linear part, but never eliminates it. In a local
quadratic surrogate that evaluates the residual on the stationary noise
buffer, the residual adds stationary covariance and raises the stationary
loss floor at every stable step size, while leaving the contraction dynamics
unchanged. Simulations of the full nonlinear recursion on quadratics and
measurements on frozen transformer gradients support each step of this
picture. Together, the results make a theoretical case for replacing
orthogonalization by response-matched momentum SGD once optimization becomes
noise-dominated.
\end{abstract}

\section{Introduction}
Muon \citep{jordan2024muon,Liuetal2025} updates each weight matrix with the
orthogonal polar factor of its momentum buffer: if the buffer has singular
value decomposition $U\Lambda V^\top$, the update direction is $UV^\top$.
Every singular direction of the buffer then moves with unit weight. When the
gradient carries signal in many directions, this equalization is the purpose
of the method. Near an optimum the situation changes. The mean gradient
becomes small, the minibatch gradient is dominated by sampling noise, and the
polar map spends its effect on noise. This paper studies what
orthogonalization does in that regime.

We organize the analysis around four questions, each answered by one
section.
\begin{enumerate}
\item \emph{What does Muon do in expectation near a noise-dominated optimum?}
Gaussian noise smooths the polar map. For a small mean gradient, the
expected update is proportional to that gradient, with a constant that
depends only on the matrix shape and the noise scale. In the mean, Muon is
therefore a gradient step (Section~\ref{sec:regimes}).
\item \emph{After matching that mean response, what fluctuation remains?}
A linear method can be calibrated to have the same local response as Muon.
A Gaussian projection shows that Muon's update equals this calibrated linear
update plus a nonlinear residual. The residual is uncorrelated with the input
noise, adds no first-order response, and contributes positive semidefinite
covariance (Section~\ref{sec:preliminaries}).
\item \emph{What does momentum do to this residual?} Momentum correlates
successive inputs to the polar map. A Hermite expansion sorts the update by
how fast it decorrelates: the linear part retains the buffer's correlation,
whereas the residual's components decorrelate faster. Momentum therefore
suppresses the residual's accumulated covariance relative to the linear
part, but cannot eliminate it (Section~\ref{sec:momentum}).
\item \emph{Does the residual affect optimization?} In a local quadratic
surrogate that keeps the full linear feedback but evaluates the residual on
the stationary noise buffer, the residual changes no contraction mode but
adds stationary covariance, so it raises the stationary loss at every stable
step size (Section~\ref{sec:local-consequences}).
\end{enumerate}
Section~\ref{sec:experiments} tests the four steps in turn, with Monte Carlo
estimates, simulations of the full nonlinear recursion on quadratics, and
measurements on frozen transformer gradients.

These answers suggest a late-stage change of optimizer. Once a matrix's
gradient is noise-dominated, Muon's nonlinearity no longer changes the local
mean dynamics but still adds fluctuation, and response-matched momentum SGD,
which keeps the buffer and rescales the learning rate, removes that
fluctuation in the surrogate. The theory proves the covariance and
surrogate-loss statements under explicit assumptions; it motivates, but does
not validate, a practical switching rule: detecting the noise-dominated regime during
training, choosing a threshold, and testing the switch in full training runs
remain open.

\paragraph{Contributions.}
\begin{enumerate}
\item \textbf{Muon is linear in the mean near a noisy optimum.} Under
isotropic Gaussian noise, the expected polar update is a gradient step with
an explicit rate, with global bounds on the cubic remainder and the
Jacobian deviation (Proposition~\ref{prop:smoothing}).
\item \textbf{Orthogonalization leaves a residual after the response is
matched.} For any nonlinear update and Gaussian input, the update splits into
a response-matched affine comparator and an orthogonal residual with positive
semidefinite covariance; for the polar map at zero signal, the one-step
covariance excess is explicit (Theorem~\ref{thm:response},
equation~\eqref{eq:one-step-ratio}).
\item \textbf{Momentum attenuates the residual but does not remove it.} For
a stationary EMA buffer we give the exact covariance accumulated over any
window, degree by degree in the Hermite expansion, and explicit ratios to
matched momentum SGD at zero signal, given by series in the Hermite weights
of the polar map, which decrease with momentum and stay
above one (Theorem~\ref{thm:momentum}, Corollary~\ref{cor:zero-signal}); the
results extend to general linear filters and Newton--Schulz maps
(Proposition~\ref{prop:momentum-filter}).
\item \textbf{In a local surrogate, the residual raises the stationary loss
floor.} In the frozen-residual quadratic surrogate, the residual adds a
positive semidefinite term to the stationary covariance, the loss gap is
strict, and switching to the response-matched comparator removes it
(Theorem~\ref{thm:switch-surrogate}). At zero signal the finite-step loss
ratio is exact, exceeds one on the whole stability interval, and identifies
where heavy momentum stops reducing the penalty
(Theorem~\ref{thm:momentum-finite-step}).
\item \textbf{The steps are tested numerically.} Monte Carlo estimates
confirm the mean response and the Hermite constants; simulations of the full
nonlinear recursion reproduce the predicted loss ratios at small steps and
the ordering over a wide range of step sizes; frozen transformer gradients
show the residual under real, anisotropic noise and the reference
implementation.
\end{enumerate}

\section{Related work}
\paragraph{Muon geometry and convergence.}
Matrix-norm and trust-region views explain why Muon's update direction is
natural for weight matrices
\citep{bernstein2024oldoptimizernewnorm,BernsteinNewhouse2025Modular,Kovalev2025Understanding,Pethicketal2025LMO},
convergence analyses bound the rate at which Muon reaches stationary
points, including with inexact polar factors
\citep{LiHong2025Note,shen2025convergenceanalysismuon,Changetal2025Beyond,KimOh2026Convergence},
and \citet{Amseletal2026} design polynomial approximations of the polar
factor for practical use.
These works address how fast Muon approaches an optimum. We study the
regime after that, where noise dominates, and ask how large Muon's
stationary fluctuations are relative to a linear method with the same local
dynamics.

\paragraph{Momentum and noise in Muon.}
\citet{li2026denoise} interpret momentum as spectral denoising that improves
the singular subspaces before orthogonalization, and
\citet{choudhury2026nesterov} analyze Nesterov Muon under heavy-tailed noise.
Our focus is complementary: we decompose the noise that remains
\emph{after} orthogonalization, relative to the response-matched linear
update, and quantify exactly how much of it momentum removes over a window.

\paragraph{Sign methods, smoothing, and projection tools.}
Sign and normalized methods change how signal and noise enter an update
\citep{bernstein2018signsgd,pmlr-v80-balles18a}. Stochastic differential
equation approximations of SGD \citep{li2017stochastic,mandt2017sgd} show,
for sign-based methods, that noise smooths the sign in the drift
\citep{Compagnonietal2025SDE}; finite-step
fluctuations of momentum SGD are analyzed by \citet{liu2021finite}. These
settings are scalar or coordinatewise and do not treat momentum-correlated
inputs to a matrix nonlinearity. Our response identity is a Gaussian
projection of Bussgang--Stein type \citep{demir2021bussgang}; we use it as a
tool. The contributions are its application to matrix orthogonalization,
the Hermite analysis of temporal accumulation under momentum, the
closed-form constants for the polar map, and the exact loss comparison in
the surrogate.
\section{Setup and the mean update under noise}\label{sec:regimes}
\paragraph{Notation.}
Let $W\in\R^{m\times n}$ be a weight matrix and $F(W)$ a differentiable
loss. For a matrix $G$, $\|G\|_F$ is the Frobenius norm, $\|G\|_*$ the
nuclear norm (the sum of the singular values), and
$\langle A,B\rangle=\tr(A^\top B)$. Write $d=mn$ and $r=\min(m,n)$. If
$G=U\Lambda V^\top$ is the compact singular value decomposition over the
positive singular values, the polar factor of $G$ is
\begin{equation}
 P(G)=\msign(G)=UV^\top,\qquad P(0)=0.
 \label{eq:polar-def}
\end{equation}
It keeps the singular vectors of $G$ and sets every nonzero singular value
to one, so $\|P(G)\|_F^2=\rank G\le r$ and $\langle G,P(G)\rangle=\|G\|_*$.
The map $P$ is odd, invariant under positive scaling, and equivariant under
orthogonal transformations: $P(AGD^\top)=AP(G)D^\top$ for $A\in O(m)$ and
$D\in O(n)$. For covariance calculations we vectorize matrices with
$\vecop:\R^{m\times n}\to\R^d$, write lowercase letters for vectorized
quantities (for example $y=\vecop Y$), and write $I_d$ for the identity on
$\R^d$.

\paragraph{Muon with and without momentum.}
Let $G_k$ be the stochastic gradient at step $k$ and $\eta_M>0$ the Muon
step size. Muon without momentum is
\begin{equation}
 W_{k+1}=W_k-\eta_M P(G_k).
 \label{eq:muon-def}
\end{equation}
Muon with momentum applies the polar map to an exponential moving average
(EMA) of the gradients,
\begin{equation}
 B_k=\beta B_{k-1}+(1-\beta)G_k,\qquad
 W_{k+1}=W_k-\eta_M P(B_k),\qquad 0\le\beta<1,
 \label{eq:muon-momentum-def}
\end{equation}
which reduces to \eqref{eq:muon-def} at $\beta=0$. Momentum SGD with the
same buffer and step size $\eta_L$ is $W_{k+1}=W_k-\eta_LB_k$. The reference
implementation of Muon uses a Nesterov form of the buffer and a five-step
Newton--Schulz approximation of $P$ \citep{jordan2024muon}; the theory is
developed for \eqref{eq:muon-momentum-def} and the exact polar map, and
Proposition~\ref{prop:momentum-filter} extends it to the reference
implementation.

\paragraph{Noise model.}
Without noise, $P(\nabla F)$ has unit singular values however small the
gradient is, so the Muon step does not shrink near a minimizer. Minibatch
noise changes this. Near a minimizer we model the stochastic gradient as
\begin{equation}
 G_k=Y+\sigma_gZ_k,
 \label{eq:noise-model}
\end{equation}
where $Y$ is the mean gradient, $\sigma_g>0$ the entrywise gradient-noise
scale, and $Z_k$ are independent matrices with independent $N(0,1)$
entries. The isotropic model yields explicit constants.
Theorem~\ref{thm:response} allows an arbitrary nondegenerate Gaussian
covariance, and Section~\ref{sec:exp-network} measures the relevant
quantities on real gradient noise.

The polar map receives $G_k$ in Muon without momentum and the buffer $B_k$
in Muon with momentum; both are the mean gradient plus Gaussian noise, with
different noise scales. We therefore state the next result for a generic
input $Y+\sigma Z$, where $\sigma$ is the noise scale of the input: $\sigma=\sigma_g$ for
\eqref{eq:muon-def} and $\sigma=\sigma_\beta$, the stationary buffer noise
scale of Section~\ref{sec:momentum}, for \eqref{eq:muon-momentum-def}.
Define the mean update and the shape constant
\begin{equation}
 \Psi_\sigma(Y)=\E P(Y+\sigma Z),\qquad
 c=c_{m,n}=\frac{\E\|Z\|_*}{d}.
 \label{eq:c-def}
\end{equation}

\begin{proposition}[Gaussian smoothing of the polar map]\label{prop:smoothing}
$\Psi_\sigma$ is $C^\infty$ and odd, $D\Psi_\sigma(0)=(c/\sigma)I_d$, and for
every $Y\in\R^{m\times n}$,
\begin{equation}
 \Bigl\|\Psi_\sigma(Y)-\frac c\sigma Y\Bigr\|_F\le\sqrt{\frac r6}\,
 \frac{\|Y\|_F^3}{\sigma^3},\qquad
 \Bigl\|D\Psi_\sigma(Y)-\frac c\sigma I_d\Bigr\|_{F\to F}
 \le\sqrt{\frac{3r}2}\,\frac{\|Y\|_F^2}{\sigma^3},
 \label{eq:regimes-linear}
\end{equation}
where $\|\cdot\|_{F\to F}$ is the operator norm induced by the Frobenius
norm. Moreover $c_{1,1}=\sqrt{2/\pi}$, $c_{1,n}\sqrt n\to1$, and
$c_{N,N}\sqrt N\to8/(3\pi)$.
\end{proposition}
The proof (Appendix~\ref{app:core-response}) differentiates the Gaussian
density rather than $P$, so rank deficiency causes no difficulty. The bounds
depend on $\|Y\|_F/\sigma$ at fixed shape: what matters is the signal
relative to the noise, not the size of the gradient alone. When this ratio
is small, the expected Muon step is a gradient step with learning rate
$\eta_Mc/\sigma$.

This is a statement about the mean drift only. It does not say that Muon
behaves like SGD: two updates can share a mean and differ in their
fluctuations. The next section asks what randomness remains once the local
response has been matched.
\section{The residual left after matching the response}\label{sec:preliminaries}\label{sec:smoothed}
We now compare one random Muon update with a linear update that has the same
local mean response. The calculation concerns a single input to the
nonlinearity and applies to Muon without momentum, or to any buffer whose
noise is Gaussian at a fixed time; temporal correlation enters in
Section~\ref{sec:momentum}.

\paragraph{Setting.}
Let $y\in\R^d$ be the signal and $g=y+\xi$ the input, with
$\xi\sim N(0,\Sigma)$ and $\Sigma\succ0$. Let $T:\R^d\to\R^p$ be the update
map; for Muon, $T(g)=\vecop P(\vecop^{-1}g)$. We state the results for a
general measurable $T$ because the argument does not use the structure of
the polar map; the same statements apply to Newton--Schulz approximations
and to anisotropic noise. Write
\begin{equation}
 \mu(y)=\E T(y+\xi),\qquad J(y)=D\mu(y)
 \label{eq:mu-J}
\end{equation}
for the mean update at signal $y$ and its derivative with respect to the
signal, holding the noise law fixed. We call $J(y)$ the \emph{local response
operator}.

\paragraph{Why match the response.}
A Muon step with input $y+\xi$ moves the parameters by $-\eta_MT(y+\xi)$,
whose mean is $-\eta_M\mu(y)$. If the signal is perturbed to $y+h$, the mean
step changes by $-\eta_M[\mu(y+h)-\mu(y)]=-\eta_MJ(y)h+o(\|h\|)$. A linear
update with step operator $A$ changes by $-Ah$. The two have the same local
deterministic dynamics exactly when $A=\eta_MJ(y)$. Comparing Muon and SGD
at equal numerical learning rates would not achieve this: differences in
their stationary behavior would then mix a difference in contraction with a
difference in fluctuation. Matching the response removes the first
difference, so that what remains isolates the effect of the nonlinearity.
This leads to the \emph{response-matched affine comparator}
\begin{equation}
 L_y(u)=\mu(y)+J(y)(u-y),\qquad
 L_y(y+\xi)=\mu(y)+J(y)\xi,
 \label{eq:affine-comparator}
\end{equation}
and the \emph{residual} $R_y=T(y+\xi)-L_y(y+\xi)$.

\begin{theorem}[Gaussian response and residual]\label{thm:response}
Let $\xi\sim N(0,\Sigma)$ in $\R^d$ with $\Sigma\succ0$, and let
$T:\R^d\to\R^p$ be measurable with at most polynomial growth. Then
$\mu(y)=\E T(y+\xi)$ is smooth, and with $J(y)=D\mu(y)$,
\begin{equation}
 \boxed{T(y+\xi)=\underbrace{\mu(y)+J(y)\xi}_{L_y(y+\xi)}+R_y,\qquad
 \E R_y=0,\quad\E[R_y\xi^\top]=0.}
 \label{eq:response-decomposition}
\end{equation}
In particular,
\begin{align}
 J(y)&=\E[(T(y+\xi)-\mu(y))\xi^\top]\Sigma^{-1},
 \label{eq:response-score}\\
 \Cov(T(y+\xi))&=J(y)\Sigma J(y)^\top+\Cov(R_y)
 \succeq J(y)\Sigma J(y)^\top.
 \label{eq:response-covariance-main}
\end{align}
\end{theorem}
The proof (Appendix~\ref{app:core-response}) differentiates the Gaussian
density, so $T$ need not be differentiable and $J(y)$ need not be square or
invertible; the identity is a Gaussian projection of Bussgang--Stein type
\citep{demir2021bussgang}. Because $R_y$ has zero mean and is uncorrelated with $\xi$, $L_y(y+\xi)$ is
the orthogonal projection of $T(y+\xi)$, in mean square, onto affine
functions of the noise, and it is the degree-one part of the Hermite
expansion used in Section~\ref{sec:momentum}. The covariance splits
accordingly. The first term of \eqref{eq:response-covariance-main} is the
covariance that the response-matched comparator must carry; the second is
the additional covariance produced by the nonlinearity. The residual has no
first-order response at $y$, but it is neither independent of the noise nor
small.

\paragraph{The polar residual at zero signal.}
For isotropic noise $\Sigma=\sigma^2I_d$ and $y=0$,
Proposition~\ref{prop:smoothing} gives $\mu(0)=0$ and $J(0)=(c/\sigma)I_d$,
so the comparator is plain SGD on the same input with learning rate
$\eta_L=\eta_Mc/\sigma$. Let $b=r/d$. Polar symmetry gives
\begin{equation}
 \vecop P(Z)=c\vecop Z+R,\qquad
 \Cov(R)=(b-c^2)I_d,\qquad
 \mathcal R_0=\frac b{c^2}=\frac{rd}{(\E\|Z\|_*)^2}>1.
 \label{eq:one-step-ratio}
\end{equation}
The ratio $\mathcal R_0$ is the one-step covariance of Muon divided by that
of response-matched SGD; $\mathcal R_0-1$ is the relative covariance added
by orthogonalization. It equals $\pi/2$ for scalars, tends to
$9\pi^2/64\approx1.39$ for large square matrices, and tends to one for long
vectors (Appendix~\ref{app:core-response}). For square matrices of every
size, orthogonalization thus adds roughly $40$--$57\%$ to the one-step
covariance. $\mathcal R_0$ is a covariance ratio, not yet a loss
ratio; Section~\ref{sec:local-consequences} supplies the connection.

The one-step decomposition does not determine the accumulated covariance.
Practical Muon uses momentum, which correlates successive inputs, and the
parameters respond to the covariance of many accumulated updates. How much
of the residual survives depends on how fast it decorrelates, which the
next section computes.
\section{How momentum acts on the residual}\label{sec:momentum}
From here on we study Muon with momentum \eqref{eq:muon-momentum-def}, with
an EMA buffer and the exact polar map; other filters and maps are treated at
the end of the section. The parameters integrate many updates, so the
relevant quantity is the covariance accumulated over a window of $K$ steps.

\paragraph{The stationary buffer.}
Fix $Y$ and $0\le\beta<1$, and let the gradients follow \eqref{eq:noise-model}.
The buffer and its stationary noise variance are
\begin{equation}
 B_k=\beta B_{k-1}+(1-\beta)(Y+\sigma_gZ_k),\qquad
 \sigma_\beta^2=\sigma_g^2\,\frac{1-\beta}{1+\beta}.
 \label{eq:momentum-stream}
\end{equation}
Start the buffer in its stationary law, $\vecop B_0\sim N(y,\sigma_\beta^2I_d)$
with $y=\vecop Y$. Then the standardized buffer noise
$X_k=(\vecop B_k-y)/\sigma_\beta$ is a stationary standard Gaussian vector
with $\Cov(X_k,X_{k+h})=\beta^{|h|}I_d$. The mean $\mu$, the response $J$,
and the comparator $L_y$ of Section~\ref{sec:preliminaries} are taken with
this buffer-noise law, $\xi\sim N(0,\sigma_\beta^2I_d)$.

\paragraph{Why a Hermite expansion.}
For a stationary centered stream $a_k$ with lag covariances
$\Gamma_h=\Cov(a_k,a_{k+h})$,
\begin{equation}
 \Cov\Bigl(\sum_{k=1}^K a_k\Bigr)
 =K\Gamma_0+\sum_{h=1}^{K-1}(K-h)(\Gamma_h+\Gamma_h^\top).
 \label{eq:sum-covariance}
\end{equation}
Theorem~\ref{thm:response} identifies the linear part and the residual at a
single time, but the accumulated covariance also depends on their lag
covariances, which a one-step calculation does not determine. For correlated
Gaussian inputs, the Hermite expansion supplies exactly this information.
Hermite polynomials are the orthogonal polynomials of the Gaussian measure,
and a component of total degree $\ell$ turns an input correlation $\rho$
into an output correlation $\rho^\ell$. The degree-one component is the
response-matched comparator; the components of degree $\ell\ge2$ make up the
residual and decorrelate as $\beta^{\ell|h|}$ rather than $\beta^{|h|}$. The
expansion thus explains, and quantifies, why momentum affects the linear
part and the residual differently.

\begin{theorem}[Accumulated covariance under EMA]\label{thm:momentum}
Under \eqref{eq:momentum-stream}, let $T:\R^d\to\R^p$ satisfy the
hypotheses of Theorem~\ref{thm:response}. For fixed $y$, expand the centered
update $f_y(x)=T(y+\sigma_\beta x)-\mu(y)=\sum_{\ell\ge1}f_{y,\ell}(x)$ into
orthogonal Hermite degrees under $X\sim N(0,I_d)$. Define the degree
covariances $C_\ell(y)=\E[f_{y,\ell}(X)f_{y,\ell}(X)^\top]\succeq0$ and, for
$K\ge1$, the window factor
\begin{equation}
 \tau_K(\rho)=1+2\sum_{h=1}^{K-1}(1-h/K)\rho^h.
 \label{eq:tau-window}
\end{equation}
Then $f_{y,1}(X)=\sigma_\beta J(y)X$, and
\begin{align}
 \Cov\!\left(\sum_{k=1}^K T(\vecop B_k)\right)
 &=K\sum_{\ell\ge1}\tau_K(\beta^\ell)C_\ell(y)\nonumber\\
 &=\Cov\!\left(\sum_{k=1}^K L_y(\vecop B_k)\right)
   +K\sum_{\ell\ge2}\tau_K(\beta^\ell)C_\ell(y).
 \label{eq:fixed-signal-cov}
\end{align}
The residual and degree-one processes are uncorrelated at every lag, and the
last term is positive semidefinite.
\end{theorem}
The proof is in Appendix~\ref{app:momentum-proof}. The window factor
$\tau_K(\rho)$ measures how much a component with lag correlation $\rho$ is
amplified by accumulation over $K$ steps. For $\beta>0$ and $K\ge2$,
$\tau_K(\beta^\ell)<\tau_K(\beta)$ for every $\ell\ge2$: accumulation
amplifies the linear part more than any residual degree, so the residual's
share of the accumulated covariance falls. It does not vanish, because
every $\tau_K$ is positive. At nonzero $y$ even degrees can occur and the
$C_\ell(y)$ need not be multiples of the identity, which is why the theorem
is stated for operators.

\paragraph{Zero signal: explicit ratios.}
At $Y=0$ the polar map is odd, so even degrees vanish and the residual
starts at degree three. Orthogonal equivariance makes each degree covariance
a multiple of the identity, $C_\ell(0)=w_\ell I_d$ (by scale invariance,
$f_0(x)=\vecop P(\vecop^{-1}x)$ does not depend on $\sigma_\beta$), with Hermite
weights $w_\ell\ge0$, $w_1=c^2$, and $\sum_{\ell\text{ odd}}w_\ell=b$. To
compare Muon with momentum SGD at equal response, scale the Muon step by
$\sigma_\beta/c$ and consider
\begin{equation}
 U_k=(\sigma_\beta/c)P(B_k),\qquad V_k=B_k,
 \label{eq:momentum-calibration}
\end{equation}
whose mean responses at zero signal both equal $I_d$; $V_k$ is the step of
matched momentum SGD.
\begin{corollary}[Zero-signal covariance ratios]\label{cor:momentum-longrun}\label{cor:zero-signal}
At $Y=0$, for every $K\ge1$,
\begin{align}
 \Cov\!\left(\vecop\sum_{k=1}^K U_k\right)
 &=\mathcal R_{\beta,K}\Cov\!\left(\vecop\sum_{k=1}^K V_k\right),
 \label{eq:momentum-finite-cov}\\
 \mathcal R_{\beta,K}
 &=1+\sum_{\substack{\ell\ge3\\\ell\text{ odd}}}
 \frac{w_\ell}{c^2}\frac{\tau_K(\beta^\ell)}{\tau_K(\beta)},
 \label{eq:momentum-finite}\\
 1<\mathcal R_{\beta,K}&\le
 1+(\mathcal R_0-1)\frac{\tau_K(\beta^3)}{\tau_K(\beta)}
 \le\mathcal R_0.
 \label{eq:momentum-finite-bound}
\end{align}
For fixed $\beta<1$ the long-window limit exists,
\begin{equation}
 \mathcal R_\beta=\lim_{K\to\infty}\mathcal R_{\beta,K}
 =1+\sum_{\substack{\ell\ge3\\\ell\text{ odd}}}
 \frac{w_\ell}{c^2}\frac{1-\beta}{1+\beta}
 \frac{1+\beta^\ell}{1-\beta^\ell},\qquad
 \lim_{\beta\uparrow1}\mathcal R_\beta
 \le1+\frac{\mathcal R_0-1}{3}.
 \label{eq:momentum-infinite}
\end{equation}
It decreases strictly with $\beta$ and satisfies
\begin{equation}
 1<\mathcal R_\beta\le
 1+\frac{1-\beta+\beta^2}{1+\beta+\beta^2}(\mathcal R_0-1).
 \label{eq:momentum-third}
\end{equation}
Moreover $K^{-1}\Cov(\vecop\sum U_k)\to\sigma_g^2\mathcal R_\beta I_d$,
while the same limit for $V_k$ is $\sigma_g^2I_d$.
\end{corollary}
The three ratios answer different questions. $\mathcal R_0$ is the one-step
covariance ratio; $\mathcal R_{\beta,K}$ is the ratio of covariances
accumulated over $K$ momentum steps; and $\mathcal R_\beta$ is its
long-window limit, the ratio of long-run noise injected per step. All are
covariance ratios at matched response. The corollary has three
consequences. First, momentum helps only through accumulation:
$\mathcal R_{\beta,1}=\mathcal R_0$ for every $\beta$, for $0<\beta<1$ the ratio
$\mathcal R_{\beta,K}$ decreases strictly in $K$ toward $\mathcal R_\beta$, and for fixed $K$ it returns to
$\mathcal R_0$ as $\beta\uparrow1$ (Appendix~\ref{app:window-proof}).
Second, the long-window ratio decreases with $\beta$ but stays above one: at
$\beta=0.95$ it is $1.09$ for scalars and $1.08$ for $64\times64$ matrices,
against $\mathcal R_0=1.57$ and $1.39$. Third, even the strongest
attenuation, $\beta\uparrow1$ after $K\to\infty$, leaves a positive excess,
at most a third of the one-step excess, and for scalars exactly
$(\pi/2)\log2\approx1.09$. Momentum reduces the residual's contribution; it
does not remove it.

\paragraph{The reference implementation.}
The argument uses only the Gaussian lag structure of the buffer and the
symmetries of the map. It therefore extends to any stationary linear filter
of the gradients, including the Nesterov buffer of the reference
implementation, and to any odd, scale-invariant, orthogonally equivariant
map, including the five-step Newton--Schulz iteration (NS5).
Proposition~\ref{prop:momentum-filter} in Appendix~\ref{app:momentum-filter}
gives the corresponding ratios, which have the same form with the filter's
lag correlations and the map's own Hermite weights. The Nesterov
correlations carry an extra factor smaller than one, so every residual
degree decorrelates faster than under EMA at the same $\beta$; NS5 has its
own constants, which we estimate numerically.

Momentum thus lowers the covariance penalty of orthogonalization without
eliminating it. Whether the remaining covariance matters for optimization
depends on how it propagates through the parameter dynamics, which we
analyze next.
\section{Stationary loss in a frozen-residual surrogate}\label{sec:local-consequences}
Accumulated covariance affects the objective only through the parameter
dynamics, and near a minimizer those dynamics feed back on themselves: the
buffer moves the parameters, the parameters change the gradient, and the
gradient re-enters the nonlinear orthogonalizer. This feedback loop through
a nonlinear map is what makes the full Muon recursion hard to analyze
directly. We therefore introduce a surrogate that keeps the linear feedback
exactly and simplifies only the nonlinear residual.

\paragraph{The surrogate.}
Near a minimizer the loss is locally quadratic. Let $x_k\in\R^d$ be the
parameter relative to a fixed expansion point, with loss gradient $y+Hx$ for
a signal $y$ and curvature $H\succ0$, and let $(z_k)_{k\in\mathbb Z}$ be
i.i.d.\ $N(0,I_d)$ gradient innovations. Split the buffer as $y+\widetilde b_k$
into the fixed signal and a centered part, and define alongside it the
noise-only buffer $n_k$ built from the same innovations:
\begin{equation}
 \widetilde b_k=\beta\widetilde b_{k-1}+(1-\beta)(Hx_k+\sigma_gz_k),\qquad
 n_k=(1-\beta)\sigma_g\sum_{j\ge0}\beta^jz_{k-j}.
 \label{eq:feedback-buffer}
\end{equation}
The centered buffer $\widetilde b_k$ contains the parameter feedback $Hx_k$;
the noise-only buffer is stationary, $n_k\sim N(0,\sigma_\beta^2I_d)$. Define
$\mu(y)$ and $J(y)$ with this noise law, and evaluate the residual of
Theorem~\ref{thm:response} on the noise-only buffer:
\begin{equation}
 e_k=T(y+n_k)-\mu(y)-J(y)n_k.
 \label{eq:frozen-residual}
\end{equation}
The response-matched comparator updates
\begin{equation}
 x_{k+1}=x_k-\eta_M[\mu(y)+J(y)\widetilde b_k],
 \label{eq:local-linear-step}
\end{equation}
which is momentum SGD with rate $\eta_L$ when $J(y)=(\eta_L/\eta_M)I_d$. The
\emph{frozen-residual surrogate} of Muon adds $e_k$ inside the brackets.
It keeps the parameter feedback through the matched linear response, the
state-transition matrix of the comparator, and the stationary temporal
statistics of the residual given by Theorem~\ref{thm:momentum}. The only
approximation is where the residual is evaluated: Muon evaluates it on the
feedback-dependent buffer $y+\widetilde b_k$, the surrogate on the
noise-only buffer. The two coincide when the
feedback term is small relative to the noise. The results below are exact
for the surrogate; Section~\ref{sec:exp-chains} tests how well it represents
the full nonlinear recursion.

\begin{theorem}[Stationary loss of the frozen-residual surrogate]
\label{thm:switch-surrogate}
Let $T$ satisfy the hypotheses of Theorem~\ref{thm:response}. After
subtracting their common stationary mean, write the surrogate and the
response-matched comparator as
$q_{k+1}^{M}=Aq_k^{M}+E_zz_k+E_ee_k$ and $q_{k+1}^{L}=Aq_k^{L}+E_zz_k$,
where $q_k\in\R^{2d}$ collects parameter and buffer coordinates. Suppose $A$
is Schur stable (all eigenvalues have modulus less than one). Then
$\E[e_kz_j^\top]=0$ at every lag, and the stationary state covariances
satisfy
\begin{equation}
 S_M=S_L+S_e,\qquad S_e\succeq0,
 \label{eq:switch-stationary}
\end{equation}
where $S_e$ is the covariance of the stable recursion forced only by $E_ee_k$
(the residual $e_k$ is stationary and square-integrable, being a fixed
measurable function of the stationary Gaussian $n_k$ with $T$ of polynomial
growth).
With $\Pi=(I_d\ \ 0)$ selecting parameter coordinates, the stationary
expected quadratic losses differ by
\begin{equation}
 \tfrac12\tr(H\Pi S_M\Pi^\top)-\tfrac12\tr(H\Pi S_L\Pi^\top)
 =\tfrac12\tr(H\Pi S_e\Pi^\top)\ge0.
 \label{eq:switch-risk}
\end{equation}
For $\eta_M>0$ the gap is strict exactly when $\E\|e_k\|^2>0$, which holds
for the exact polar map at every fixed signal. If the surrogate is switched
to the comparator at a fixed finite time, retaining a square-integrable
state and using the same future innovations, its covariance converges to
$S_L$ and its expected quadratic loss to the comparator's stationary loss.
\end{theorem}
The proof is in Appendix~\ref{app:switch-proof}. The two recursions share
the transition matrix $A$, so they have identical contraction modes; the
residual changes only the forcing. Because that forcing is uncorrelated with
the gradient innovations at every lag, its stationary covariance adds to
that of the comparator, and the loss gap equals the loss produced by the
residual alone. (Uncorrelatedness at a single lag would not suffice.) At
nonzero $y$ the common stationary mean may differ from the minimizer, and
\eqref{eq:switch-risk} compares the fluctuations around that shared bias.

\paragraph{Exact loss ratio at zero signal.}
At $y=0$ with the exact polar map, $\mu(0)=0$ and $J(0)=(c/\sigma_\beta)I_d$,
so the comparator is EMA momentum SGD with the matched rate
\begin{equation}
 \eta_L=\eta_M\,c/\sigma_\beta,
 \label{eq:matched-rate}
\end{equation}
and the surrogate is $x_{k+1}=x_k-\eta_L[\widetilde b_k+(\sigma_\beta/c)e_k]$.
In this case the loss ratio can be computed on the whole stability interval.
Let $v_\beta=(1-\beta)/(1+\beta)$ and, for $0\le q<\beta$ and
$0\le t<2/v_\beta$,
\begin{align}
 D_\beta(q,t)&=(1-q)(1-\beta q)+(1-\beta)tq,\nonumber\\
 L_\beta(q,t)&=1+
 \frac{2(\beta-q)\{t(1+\beta^2q)-(1+\beta)(1-\beta q)\}}
 {(1+\beta)^2D_\beta(q,t)},\qquad
 \mathcal R^{\rm lin}_\beta(t)=1+\sum_{\substack{\ell\ge3\\\ell\text{ odd}}}
 \frac{w_\ell}{c^2}L_\beta(\beta^\ell,t).
 \label{eq:finite-step-penalty}
\end{align}
Here $t$ plays the role of the effective step $\eta_L\lambda$ of a Hessian
mode with eigenvalue $\lambda$, $L_\beta(\beta^\ell,t)$ is the stationary
variance of that mode driven by the degree-$\ell$ residual relative to the
variance driven by the Gaussian buffer noise, and $\mathcal R^{\rm lin}_\beta(t)$
is the resulting per-mode loss penalty.
\begin{theorem}[Exact finite-step loss ratio of the surrogate]\label{thm:momentum-finite-step}
Let $y=0$, let $T$ be the exact polar map, $0<\beta<1$,
$H=U\diag(\lambda_i)U^\top\succ0$, and $0<\eta_L\lambda_{\max}<2/v_\beta$.
Write $\Sigma_L=\Pi S_L\Pi^\top$ and $\Sigma_M=\Pi S_M\Pi^\top$ for the
parameter blocks of the stationary covariances of matched momentum SGD and of
the surrogate. Then
\begin{equation}
 \Sigma_L=\eta_L\sigma_g^2(2H-\eta_Lv_\beta H^2)^{-1},\qquad
 \Sigma_M=U\diag\!\left(
 \frac{\eta_L\sigma_g^2\,\mathcal R^{\rm lin}_\beta(\eta_L\lambda_i)}
 {2\lambda_i-\eta_Lv_\beta\lambda_i^2}\right)U^\top.
 \label{eq:finite-step-covariance}
\end{equation}
The function $\mathcal R^{\rm lin}_\beta$ is continuous and strictly
increasing on $[0,2/v_\beta)$, with $\mathcal R^{\rm lin}_\beta(0)=\mathcal R_\beta$,
and it crosses $\mathcal R_0$ at a unique $t_*$ with
\begin{equation}
 \frac{(1+\beta)(1-\beta^4)}{1+\beta^5}\le t_*<1+\beta.
 \label{eq:finite-step-threshold}
\end{equation}
Consequently the ratio of stationary excess losses
$\tr(H\Sigma_M)/\tr(H\Sigma_L)$ is a weighted average of the values
$\mathcal R^{\rm lin}_\beta(\eta_L\lambda_i)$; it lies in
$[\mathcal R^{\rm lin}_\beta(\eta_L\lambda_{\min}),\mathcal R^{\rm lin}_\beta(\eta_L\lambda_{\max})]$,
is strictly greater than one, and tends to $\mathcal R_\beta$ as
$\eta_L\downarrow0$. For $\beta=0$ the ratio equals $\mathcal R_0$ at every
stable step.
\end{theorem}
The proof is in Appendix~\ref{app:momentum-finite-step}. Three consequences
follow, all for the surrogate. First, the long-window covariance ratio
$\mathcal R_\beta$ of Section~\ref{sec:momentum} is exactly the small-step
loss ratio, so the covariance ratios become loss ratios.
Second, the loss ratio exceeds one at every stable step, so lowering the
learning rate cannot bring the loss ratio to one: at any rate, matched
momentum SGD reaches a lower stationary loss with the same contraction. Third, momentum
reduces the no-momentum penalty $\mathcal R_0$ only for modes with
$\eta_L\lambda_i<t_*$. Since $t_*<1+\beta$ while stability allows
$\eta_L\lambda_i$ up to $2/v_\beta$, modes with
$1+\beta\le\eta_L\lambda_i<2/v_\beta$ are stable but pay more than
$\mathcal R_0$. As $\beta\uparrow1$ the guaranteed range
$(1+\beta)(1-\beta^4)/(1+\beta^5)\sim4(1-\beta)$ shrinks while the stability
limit grows like $4/(1-\beta)$, so large effective steps with heavy momentum
lose the benefit of momentum long before they lose stability.

\paragraph{From the surrogate to a switching rule.}
The surrogate analysis establishes that, once the response is matched, the
nonlinear residual adds stationary covariance without improving the local
first-order dynamics, and that switching to the comparator, keeping the
momentum buffer, replacing $P(B_k)$ by $B_k$, and setting the rate by
\eqref{eq:matched-rate}, removes that covariance. It motivates a candidate
switching rule for training, whose use requires three conditions. The
gradient must be noise-dominated: \eqref{eq:regimes-linear} guarantees a
dominant linear response when $\sqrt r\,\|Y\|_F^2\ll c\,\sigma_\beta^2$, and
Section~\ref{sec:exp-mean} shows that the linear response holds well beyond
this bound. The response $J$ must be close to a multiple of
the identity, since momentum SGD has a scalar response;
Appendix~\ref{app:scalarity} quantifies this approximation. And the local
dynamics must be approximately stationary. To locate the regime we report
\begin{equation}
 \delta_\beta=\frac{\|Y\|_F}{\sigma_\beta\sqrt d},
 \label{eq:delta-beta}
\end{equation}
the ratio of the mean norm to the root-mean-square norm of the buffer noise.
Values much smaller than one indicate the small-signal regime of
Proposition~\ref{prop:smoothing}; values of order one mark a transition
region in which the linear approximation may be inaccurate. The theory does
not supply a universal threshold on $\delta_\beta$.
\section{Experiments}\label{sec:experiments}
The experiments test the four steps of the argument in order: the mean
response (Section~\ref{sec:exp-mean}), the residual under momentum
(Section~\ref{sec:exp-momentum}), its effect on the stationary loss of the
full nonlinear recursion (Section~\ref{sec:exp-chains}), and its presence
under real gradient noise (Section~\ref{sec:exp-network}). The first three
use synthetic Gaussian noise; the last uses stored minibatch gradients of a
small transformer. No experiment trains a network with a switched
optimizer. Methods and the full table of Gaussian constants are in
Appendix~\ref{app:main-methods}; intervals are 95\% intervals over the
stated batches or chains.

\subsection{The mean update is a scaled gradient step}\label{sec:exp-mean}
We first test whether the expected polar update follows the linear response
$(c/\sigma)Y$ of Proposition~\ref{prop:smoothing} at small signal. We estimate the mean update at
$\|Y\|_F=0.1\sigma$ by Monte Carlo and project it on $Y$. The measured slope
agrees with $c/\sigma$ within $0.8\%$ over a $32$-fold range of noise scales
and over matrix sizes $N=4$ to $64$, and the full mean matrix agrees with
$(c/\sigma)Y$ entry by entry (Figure~\ref{fig:exp-mean}).

\begin{figure}[t]
\centering
\includegraphics[width=\linewidth]{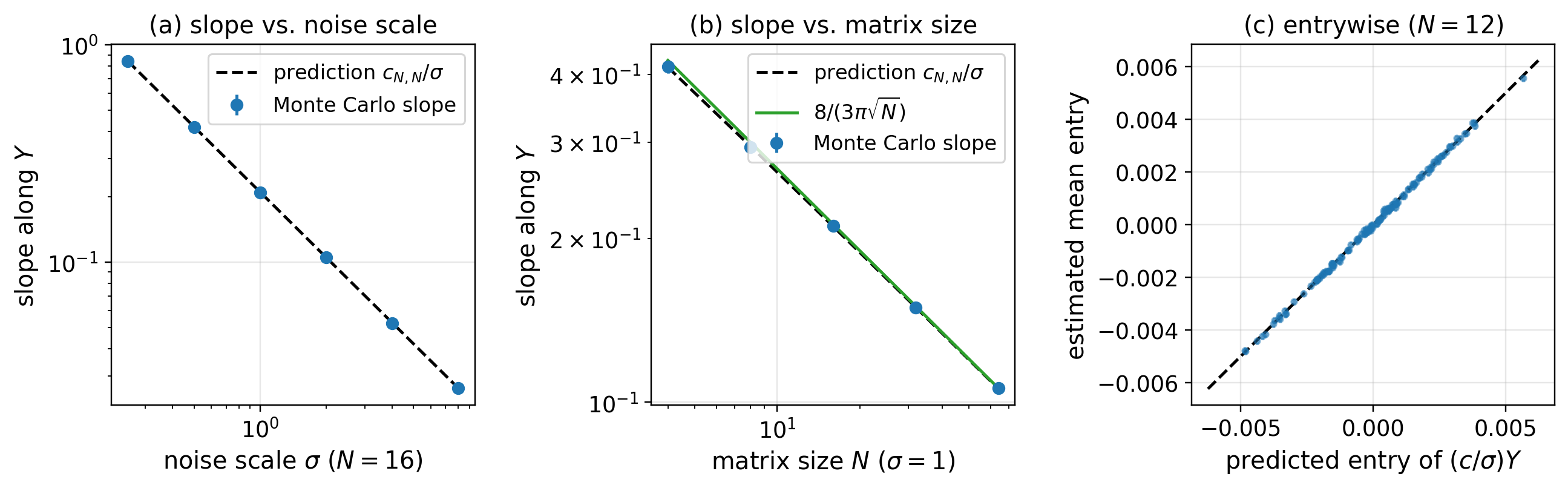}
\caption{At small signal the mean polar update is the predicted scaled
gradient step. Monte Carlo estimates of $\E P(Y+\sigma Z)$ at
$\|Y\|_F=0.1\sigma$ for $N\times N$ matrices: (a) slope along $Y$ against the
noise scale ($N=16$), (b) slope against the matrix size ($\sigma=1$), with the
asymptote $8/(3\pi\sqrt N)$, and (c) estimated against predicted entries
($N=12$, $\sigma=1$). Dashed lines are the independent predictions
$c_{N,N}/\sigma$; error bars are 95\% intervals.}
\label{fig:exp-mean}
\end{figure}

\subsection{Momentum attenuates the residual but does not remove it}\label{sec:exp-momentum}
Next we evaluate the covariance ratios of Corollary~\ref{cor:zero-signal}.
We estimate the residual
correlation between pairs of correlated Gaussian matrices, fit nonnegative
odd-degree Hermite weights, and substitute them into the corollary. For
scalars the ratio is known in closed form (Appendix~\ref{app:momentum-scalar}); the estimate is $1.0913\pm0.0009$
against the exact $1.0920$ at $\beta=0.95$. For $64\times64$ exact polar
updates the long-window ratio falls from $\mathcal R_0=1.394$ to
$\mathcal R_{0.95}=1.080$ (Figure~\ref{fig:exp-momentum}a). This removes
$80\%$ of the excess, more than the degree-three bound guarantees, because
the higher residual degrees decorrelate even faster; the ratio stays above
one at every momentum. The same reduction holds for NS5 and the Nesterov
filter, with their own constants (Figure~\ref{fig:exp-momentum}b). For the
combination used by the reference implementation, NS5 with the Nesterov
filter at $\beta=0.95$, the $64\times64$ ratio falls from $1.509$ to $1.067$.

\begin{figure}[t]
\centering
\includegraphics[width=\linewidth]{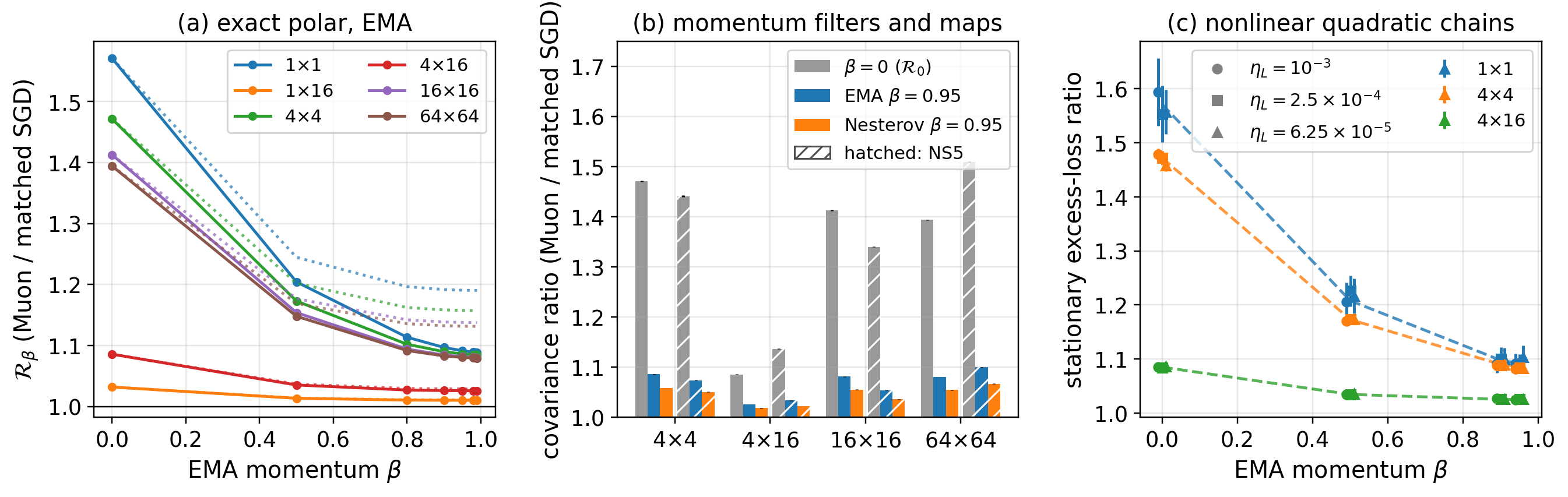}
\caption{Momentum substantially reduces the covariance excess of
orthogonalization but does not eliminate it, and the prediction carries over
to the stationary loss of the full nonlinear recursion. (a) Long-window
ratio $\mathcal R_\beta$ of Muon to matched momentum SGD for exact polar
updates under EMA (solid), with the degree-three bound
\eqref{eq:momentum-third} (dotted). (b) Covariance ratio without momentum
($\beta=0$, gray, equal to $\mathcal R_0$) and with EMA or Nesterov momentum at
$\beta=0.95$ (long-window limit), for exact polar (solid) and NS5 (hatched):
momentum removes most, but not all, of the excess for both maps and both
filters. (c) Stationary
excess-loss ratio of nonlinear exact-polar EMA Muon to matched momentum SGD
on quadratics, at three matched learning rates $\eta_L$, against the
prediction $\mathcal R_\beta$ (dashed).}
\label{fig:exp-momentum}
\end{figure}

\subsection{The residual raises the stationary loss of the full recursion}\label{sec:exp-chains}
The surrogate of Section~\ref{sec:local-consequences} evaluates the residual
on the noise-only buffer, whereas Muon evaluates it on the feedback-dependent
buffer. To test whether the predicted loss gap survives this difference, we
run nonlinear Muon and matched
momentum SGD on dense positive-definite quadratics with Hessian eigenvalues
in $[0.5,2]$, with the same filter, paired Gaussian innovations, 32 paired
chains, and three matched learning rates. Across 24 configurations (three
shapes, EMA and Nesterov filters, exact polar and NS5), every measured ratio
of stationary excess losses lies within $1.5\%$ of the predicted
$\mathcal R_\beta$, and at the smallest step 22 of the 24 intervals contain
it. Figure~\ref{fig:exp-momentum}c shows the exact-polar EMA chains;
Table~\ref{tab:momentum-chains} in Appendix~\ref{app:main-methods} lists all
24 configurations. At the smallest step, for $4\times4$ exact polar updates,
the ratio falls from $1.458\pm0.011$ without momentum to $1.084\pm0.005$ at EMA
$0.95$, against predictions $1.472$ and $1.086$.

\begin{table}[t]
\centering
\small
\begin{tabular}{rcccc}
\toprule
 & \multicolumn{2}{c}{exact polar} & \multicolumn{2}{c}{NS5}\\
\cmidrule(lr){2-3}\cmidrule(lr){4-5}
$\eta_L\lambda_{\max}$ & nonlinear Muon & surrogate & nonlinear Muon & surrogate\\
\midrule
$0.02$ & $1.081\pm0.001$ & $1.088$ & $1.055\pm0.001$ & $1.060$ \\
$0.05$ & $1.078\pm0.001$ & $1.094$ & $1.052\pm0.001$ & $1.065$ \\
$0.1$ & $1.074\pm0.001$ & $1.104$ & $1.049\pm0.001$ & $1.072$ \\
$0.2$ & $1.070\pm0.001$ & $1.123$ & $1.045\pm0.001$ & $1.085$ \\
$0.4$ & $1.064\pm0.002$ & $1.159$ & $1.041\pm0.002$ & $1.111$ \\
$0.8$ & $1.062\pm0.002$ & $1.224$ & $1.038\pm0.002$ & $1.159$ \\
$1.6$ & $1.066\pm0.003$ & $1.333$ & $1.042\pm0.003$ & $1.242$ \\
$3.2$ & $1.082\pm0.004$ & $1.504$ & $1.058\pm0.004$ & $1.379$ \\
$6.4$ & $1.128\pm0.007$ & $1.750$ & $1.109\pm0.008$ & $1.592$ \\
\midrule
$\mathcal R_{0.9}$ (small-step limit) & \multicolumn{2}{c}{$1.084$} & \multicolumn{2}{c}{$1.057$}\\
$\mathcal R_0$ (no momentum) & \multicolumn{2}{c}{$1.413$} & \multicolumn{2}{c}{$1.340$}\\
\bottomrule
\end{tabular}
\caption{The stationary-loss gap persists at every stable step size. Ratio
of stationary losses of nonlinear EMA Muon ($\beta=0.9$) to matched momentum
SGD on a $16\times16$ matrix quadratic, over a $320$-fold range of effective
steps (95\% intervals over 64 paired chains), with the frozen-residual
prediction of Theorem~\ref{thm:momentum-finite-step} (NS5: the same formula
with the NS5 Hermite weights). The measured ratio stays between $1.04$ and
$1.13$, above one and well below $\mathcal R_0$; the surrogate is accurate at
small steps and conservative at large ones.}
\label{tab:exp-steps}
\end{table}

The gap is not confined to small steps. On a $16\times16$ matrix quadratic
$\tfrac12\tr(E^\top PEQ)$ with EMA $\beta=0.9$ and 64 paired chains, the
stationary loss of nonlinear Muon exceeds that of matched momentum SGD at
every tested step, over a $320$-fold range of effective steps
$\eta_L\lambda_{\max}$ (Table~\ref{tab:exp-steps}). The ratio is
$1.081\pm0.001$ at the smallest step for exact polar, against
$\mathcal R_{0.9}=1.084$, and stays between $1.04$ and $1.13$ for both maps,
well below $\mathcal R_0$. As the step grows, the surrogate overpredicts the
ratio, so the frozen-residual approximation is conservative there, but the
ordering it predicts persists in the full recursion.

\subsection{Frozen transformer gradients}\label{sec:exp-network}
\begin{figure}[!t]
\centering
\includegraphics[width=\linewidth]{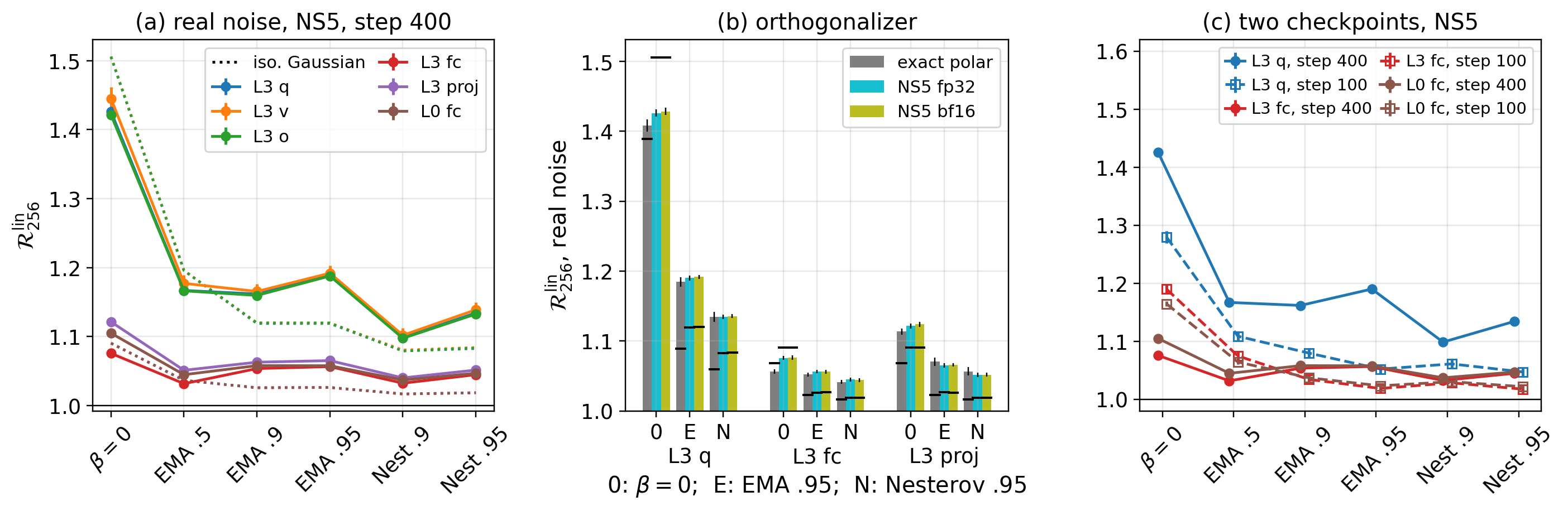}
\caption{The residual persists under real transformer gradient noise and in
the reference implementation, and momentum reduces it only partly.
Accumulated covariance ratio $\mathcal R^{\rm lin}_{256}$ for six hidden
matrices of a frozen transformer (effective batch 64). (a) Real noise
(solid) against isotropic Gaussian controls (dotted), NS5 in float32.
(b) Exact polar and NS5 in float32 and bfloat16 for three of the matrices at
$\beta=0$, EMA $0.95$, and Nesterov $0.95$; black lines mark the isotropic
controls. (c) The main
checkpoint (step 400, solid) and an earlier one (step 100, dashed). Error
bars are 95\% intervals conditional on the stored gradient pool.}
\label{fig:exp-network}
\end{figure}
Real gradient noise is anisotropic, non-Gaussian, and has a nonzero mean,
and the reference implementation uses NS5 in bfloat16 with a Nesterov
buffer. To measure the residual in this setting without the confound of a moving
network, we freeze a six-layer character-level transformer trained on Tiny
Shakespeare, store $8192$ minibatch gradients for six hidden matrices, and
resample stationary gradient streams with effective batch $64$. Step $400$
is the main checkpoint and step $100$ an earlier stage of training.

Let $M_k$ be the filtered matrix supplied to the orthogonalizer,
$m_k=\vecop M_k$, $T$ the vectorized orthogonalizer, and $J$ the response of
$\E T(m_k)$ to a persistent shift of the gradient mean with the noise law
held fixed. We report
\begin{equation}
 \mathcal R^{\rm lin}_K=
 \frac{\tr\Cov(\sum_{k=1}^K T(m_k))}
 {\tr\Cov(\sum_{k=1}^K Jm_k)},
 \label{eq:network-main-ratio}
\end{equation}
the accumulated covariance of the update divided by that of the
response-matched linear update $Jm_k$. A value of one would mean that the
orthogonalizer adds no accumulated covariance beyond its matched linear
response; values above one measure the remaining nonlinear fluctuation. For
zero-mean isotropic Gaussian noise under EMA with exact polar updates,
$\mathcal R^{\rm lin}_K=\mathcal R_{\beta,K}$, and an isotropic Gaussian
control run through the same pipeline reproduces the zero-signal theory
within $0.01$.

At step $400$ and $K=256$, the attention-matrix ratios fall from
$1.42$--$1.45$ without momentum to $1.19$ under EMA $0.95$ and $1.13$--$1.14$
under the reference Nesterov filter at $0.95$ (Figure~\ref{fig:exp-network}a).
Exact polar, NS5 in float32, and NS5 in bfloat16 agree within $0.03$
(Figure~\ref{fig:exp-network}b). Gaussian noise with the measured covariance
and zero mean gives ratios close to the isotropic prediction at high
momentum, and adding the measured mean reproduces the real-noise values
within $0.012$ for every layer at $\beta\ge0.9$ (Table~\ref{tab:network-ratios}).
For the attention matrices, the excess over the zero-signal theory at high
momentum is therefore attributable mainly to the mean gradient; for the
rectangular MLP matrices, anisotropy also contributes. The indicator $\delta_\beta$ of
\eqref{eq:delta-beta} is $1.0$--$1.7$ at step $400$; at step $100$ it is
$8$--$11$ and the high-momentum ratios are $1.02$--$1.05$
(Figure~\ref{fig:exp-network}c). The residual therefore persists on real
gradients and in the reference implementation, and its size differs between
two checkpoints with very different buffer signal-to-noise ratios; the two
checkpoints also differ in other respects, so this comparison does not
isolate the effect of $\delta_\beta$. None of these measurements tests the
effect of a switch on training.

\section{Discussion}\label{sec:limitations}
\paragraph{What the theory establishes.}
Under Gaussian noise, Muon's expected update near a noise-dominated optimum
is a scaled gradient step. After this response is matched, every update
contains a nonlinear residual that is uncorrelated with the noise and adds
positive semidefinite covariance; this decomposition is exact for any
Gaussian input. Under a stationary EMA or linear-filter buffer, the Hermite
expansion gives the exact accumulated covariance: momentum suppresses the
residual relative to the linear part but never removes it, and at zero
signal the covariance ratios are explicit series in the Hermite weights. In the frozen-residual
surrogate, the residual raises the stationary loss at every stable step
without changing the contraction modes, and switching to the response-matched
comparator removes the excess.

\paragraph{What is model-dependent or approximate.}
The response theorem assumes a Gaussian input law, and the temporal theorem
a stationary buffer driven by independent Gaussian innovations. The explicit
constants are for zero signal and isotropic noise; at nonzero signal the
decomposition still holds, but the response and the degree covariances are
operators that we do not evaluate in closed form. The loss results are
exact for the surrogate, which evaluates the residual on the noise-only
buffer rather than the feedback-dependent one. Simulations of the full
nonlinear recursion agree with the surrogate at small steps and preserve its
ordering at large ones, but we do not prove an approximation theorem.
Finally, replacing the comparator by ordinary momentum SGD requires an
approximately scalar response (Appendix~\ref{app:scalarity}).

\paragraph{What a practical switch still requires.}
Turning the analysis into a training procedure requires detecting the
noise-dominated regime reliably, choosing a threshold on an indicator such
as $\delta_\beta$, estimating the matched learning-rate scale robustly, and
validating the switch in full training runs at scale. The frozen-gradient
measurements show that the residual is present in the reference
implementation, but they are not a substitute for such runs. The
implication is not that Muon should be replaced by SGD in general. Once
Muon's local mean response has become effectively linear, continuing the
nonlinear orthogonalization carries an additional covariance cost;
determining when this happens during large-scale training is the main
practical question the theory leaves open.

\bibliographystyle{plainnat}
\bibliography{references}
\clearpage
\appendix
\section{Proofs for Sections~\ref{sec:regimes} and~\ref{sec:preliminaries}}\label{app:core-response}
This appendix proves Theorem~\ref{thm:response} and Proposition~\ref{prop:smoothing}, and
derives the zero-signal residual \eqref{eq:one-step-ratio}.

\paragraph{General Gaussian projection.}
Polynomial growth of $T$ makes its convolution with a nondegenerate Gaussian
smooth and square-integrable at every fixed shift. Differentiating the
Gaussian density, rather than $T$, gives for $H\in\R^d$
\[
 D\mu(y)[H]=\E[T(y+\xi)\,\xi^\top\Sigma^{-1}H].
\]
Centering $T$ does not change this identity, since $\E\xi=0$. Therefore
$\E[(T(y+\xi)-\mu(y))\xi^\top]=J(y)\Sigma$ and
$\E[R_y\xi^\top]=0$. Expanding the covariance yields
\eqref{eq:response-covariance-main}. This proves
Theorem~\ref{thm:response} for any output dimension $p$; no inverse of
$J(y)$ is used. This is an orthogonal projection in Gaussian $L^2$, not
an independence or stochastic-dominance assertion
\citep{demir2021bussgang}. The argument differentiates a Gaussian density,
so it needs no pointwise derivative of the polar map at deficient rank.

\paragraph{Smoothness and response.}
The polar factor is measurable and bounded by $\sqrt r$ in Frobenius norm.
Every derivative of a Gaussian density is integrable, so its convolution
$\Psi_\sigma$ is smooth. Oddness follows from $P(-G)=-P(G)$ and the symmetry of
$Z$. Differentiating the shifted density at the origin gives
\[
 D\Psi_\sigma(0)[H]=\sigma^{-1}\E[P(Z)\langle Z,H\rangle].
\]
Let $C=\E[\vecop P(Z)(\vecop Z)^\top]$. Independent row and column sign
changes preserve the Gaussian law and commute with $P$. They force
$C_{(i,j),(k,l)}=0$ whenever $(i,j)\ne(k,l)$: flip row $i$ if $i\ne k$,
and column $j$ otherwise. Row and column permutations make all diagonal
entries equal. Thus $C=cI_d$, where
\[
 dc=\tr C=\E\langle P(Z),Z\rangle=\E\|Z\|_*.
\]
This proves the response formula.

\paragraph{Small-signal bounds.}
Since $P$ is odd, $\Psi_\sigma$ is odd and $D^2\Psi_\sigma(0)=0$.
For matrices $A,B,C$, differentiation of the Gaussian density three times
and the Gaussian Wick isometry give
\[
 \|D^3\Psi_\sigma(Y)[A,B,C]\|_F
 \le \frac{\sqrt{6r}}{\sigma^3}\|A\|_F\|B\|_F\|C\|_F.
\]
Indeed the scalar third-order score is the degree-three Wick product of
$\langle Z,A\rangle,\langle Z,B\rangle,\langle Z,C\rangle$; its variance
is at most $3!\|A\|_F^2\|B\|_F^2\|C\|_F^2$, and $\|P\|_F\le\sqrt r$.
Taylor's formula along the line from $0$ to $Y$ gives
\[
 \left\|\Psi_\sigma(Y)-\frac c\sigma Y\right\|_F
 \le\frac{\sqrt{6r}}6\frac{\|Y\|_F^3}{\sigma^3},\qquad
 \left\|D\Psi_\sigma(Y)-\frac c\sigma I_d\right\|_{F\to F}
 \le\frac{\sqrt{6r}}2\frac{\|Y\|_F^2}{\sigma^3}.
\]
These are the bounds \eqref{eq:regimes-linear}.

\paragraph{Covariance and projection.}
The same sign and permutation argument makes the second-moment matrix
$\E[\vecop P(Z)\vecop P(Z)^\top]$ a scalar identity. Gaussian $Z$ has rank $r$ almost surely, so
$\|P(Z)\|_F^2=r$ almost surely, and the scalar is $b=r/d$. Its mean is zero
by oddness. Set $R=\vecop P(Z)-c\vecop Z$. Then $\E R=0$,
$\E[R(\vecop Z)^\top]=C-cI_d=0$, and expanding the covariance gives
$\Cov(R)=(b-c^2)I_d$. Multiplication by $\sigma/c$ proves the calibrated ratio.
Finally,
\[
 \E\|Z\|_*\le\sqrt r\,\E\|Z\|_F
 <\sqrt r\,(\E\|Z\|_F^2)^{1/2}=\sqrt{rd},
\]
where strictness follows because $\|Z\|_F$ is nonconstant. Hence
$\mathcal R_0>1$ in every finite dimension. This is an orthogonal projection
in the space of square-integrable random vectors, consistent with the
Gaussian Bussgang decomposition \citep{demir2021bussgang}.

\paragraph{Shape limits used in the main text.}
For $m=n=1$, $c=\E|Z|=\sqrt{2/\pi}$. For $m=1,n=d$,
$c=d^{-1}\E\chi_d$ and $\E\chi_d/\sqrt d\to1$, so $\mathcal R_0\to1$.
For $m=n=N$, the singular-value law of $Z/\sqrt N$ converges to the
quarter-circle density $\pi^{-1}\sqrt{4-x^2}$ on $[0,2]$
\citep{marchenko1967distribution}. Thus
\[
 \frac{\E\|Z\|_*}{N^{3/2}}\longrightarrow
 \frac1\pi\int_0^2x\sqrt{4-x^2}\,dx=\frac8{3\pi},
 \qquad \mathcal R_0\longrightarrow\frac{9\pi^2}{64}.
\]
To justify convergence of the expectation, use the eigenvalue law of
$ZZ^\top/N$. Above a truncation $L$, $\sqrt{x}\le x/\sqrt L$, and the
expected first moment of this law equals one. Bounded convergence handles
the bounded truncation; the tail is at most $1/\sqrt L$, uniformly in $N$.
These limits do not give an upper bound on finite-dimensional ratios.
\section{Proofs for Section~\ref{sec:momentum}}\label{app:momentum-proofs}
\subsection{Proof of Theorem~\ref{thm:momentum} and Corollary~\ref{cor:momentum-longrun}}
\label{app:momentum-proof}
\paragraph{Fixed-signal Hermite identity.}
Let $H_\alpha$ be the product of normalized probabilists' Hermite
polynomials, indexed by multi-indices $\alpha\in\mathbb N^d$.
Since $f_y$ is centered and square-integrable, write
$f_y=\sum_{\ell\ge1}f_{y,\ell}$ in Gaussian $L^2$, where
$f_{y,\ell}=\sum_{|\alpha|=\ell}a_\alpha H_\alpha$ and
$C_\ell(y)=\sum_{|\alpha|=\ell}a_\alpha a_\alpha^\top$.
The score identity of Theorem~\ref{thm:response} yields
$\E[f_y(X)X^\top]=\sigma_\beta J(y)$, so
$f_{y,1}(X)=\sigma_\beta J(y)X$ and
$C_1(y)=\sigma_\beta^2J(y)J(y)^\top$.
For jointly standard Gaussian $X,X'$ with cross-covariance $\rho I_d$,
the Hermite generating function gives
$\E[H_\alpha(X)H_\gamma(X')]=\boldsymbol1_{\alpha=\gamma}\rho^{|\alpha|}$
(proved as \eqref{eq:hermite-correlation} below).
Consequently
$\E[f_{y,\ell}(X)f_{y,j}(X')^\top]
=\boldsymbol1_{\ell=j}\rho^\ell C_\ell(y)$.
The identity holds for the infinite expansion by $L^2$ approximation and
Cauchy--Schwarz. Substituting $\rho=\beta^{|h|}$, summing the $K$ diagonal
and $2(K-h)$ off-diagonal terms, proves
\eqref{eq:fixed-signal-cov}. All $C_\ell(y)$ are PSD and
$\tau_K(\beta^\ell)\ge1$ for $\beta\ge0$, proving the ordering. The same
calculation shows that the degree-one and residual processes are
uncorrelated at every lag. It does not assert their independence.

\paragraph{Zero-signal specialization.}
Let $f(z)=\vecop\msign(Z)$, where $z=\vecop Z\sim N(0,I_d)$. Then
$\|f(z)\|^2=r$ almost surely and $f(-z)=-f(z)$. Let $\mathcal H_\ell$ be the
span of products of normalized probabilists' Hermite polynomials with total
degree $\ell$, and let $f_\ell$ be the coordinatewise orthogonal projection of
$f$ onto $\mathcal H_\ell$. Completeness of the Gaussian Hermite basis and
oddness give $f=\sum_{\ell\text{ odd}}f_\ell$ in $L^2$.

\paragraph{Isotropy of each degree.}
For orthogonal $A\in O(m)$, $D\in O(n)$, the transformation $Z\mapsto AZD^\top$
is orthogonal in vectorized coordinates and
$\msign(AZD^\top)=A\msign(Z)D^\top$. Total Hermite degree is preserved by
orthogonal transformations, so each $f_\ell$ inherits this equivariance.
Independent row and column sign changes force distinct matrix coordinates to
have zero covariance; row and column permutations equate their variances.
Consequently,
\begin{equation}
 \E[f_\ell(z)f_\ell(z)^\top]=w_\ell I_d,\qquad
 w_\ell\ge0,\qquad \sum_{\ell\text{ odd}}w_\ell=r/d=b.
 \label{eq:hermite-weights}
\end{equation}
The same invariance makes $\E[f(z)z^\top]$ a scalar identity. Taking its trace
and using $\langle\msign(Z),Z\rangle=\|Z\|_*$ gives
$\E[f(z)z^\top]=cI_d$. Thus $f_1(z)=cz$ and $w_1=c^2$. Some higher-degree
weight is positive: otherwise $f(z)=cz$ almost surely, contradicting its fixed
nonzero norm and the nonconstant norm of a Gaussian vector. In particular $b>c^2$.

\paragraph{Lag covariance.}
For jointly standard Gaussian $z,z'$ with $\E[zz'^\top]=\rho I_d$, the Hermite
generating function gives, for normalized product polynomials $H_\alpha$,
\begin{equation}
 \E[H_\alpha(z)H_\gamma(z')]
 =\boldsymbol 1_{\alpha=\gamma}\rho^{|\alpha|}.
 \label{eq:hermite-correlation}
\end{equation}
Indeed, the expectation of
$\exp(t^\top z-\|t\|^2/2)\exp(u^\top z'-\|u\|^2/2)$ equals
$\exp(\rho t^\top u)$; equating coefficients gives the identity. Applying it
to finite Hermite expansions and passing to the $L^2$ limit by Cauchy--Schwarz
yields
\begin{equation}
 \E[f_\ell(z)f_j(z')^\top]
 =\boldsymbol1_{\ell=j}\rho^\ell w_\ell I_d.
 \label{eq:chaos-cross-cov}
\end{equation}
At $Y=0$, $X_k=\vecop B_k/\sigma_\beta$ is stationary standard Gaussian with
cross-covariance $\beta^{|h|}I_d$ at lag $h$. Homogeneity of the polar map and
\eqref{eq:chaos-cross-cov} therefore give
\begin{equation}
 \Cov(\vecop\msign(B_k),\vecop\msign(B_{k+h}))
 =\sum_{\ell\text{ odd}}w_\ell\beta^{\ell|h|}I_d.
\end{equation}
For a stationary sequence, the covariance of its sum over $K$ times is
$K\Gamma_0+\sum_{h=1}^{K-1}(K-h)(\Gamma_h+\Gamma_h^\top)$.
Substitution proves \eqref{eq:momentum-finite-cov}--\eqref{eq:momentum-finite},
since the comparator covariance is $K\sigma_\beta^2\tau_K(\beta)I_d$.
For $\ell\ge3$, $\beta^\ell\le\beta^3$, and $\tau_K$ has nonnegative
coefficients. Summing the resulting inequalities proves
\eqref{eq:momentum-finite-bound}. When $0<\beta<1$ and $K\ge2$,
$\tau_K(\beta^\ell)<\tau_K(\beta)$ for every $\ell>1$; the positive residual
mass gives strictness. Positivity of every $\tau_K$ gives
$\mathcal R_{\beta,K}>1$ including at $K=1$ and $\beta=0$.

\paragraph{Calibration.}
Gaussian convolution gives the score derivative
$D_Y\E\msign(Y+\sigma_\beta Z)|_{Y=0}=(c/\sigma_\beta)I_d$.
This differentiation acts on the Gaussian density and does not differentiate
the polar map pointwise at deficient rank. Hence both directions in
\eqref{eq:momentum-calibration} have derivative $I_d$. The stationary law is
shifted with $Y$ in this calculation; it is not a claim about a transient
response from an initialization independent of $Y$.

\paragraph{Infinite window and monotonicity.}
For fixed $\beta<1$, the lag-covariance series is absolutely summable, since its
operator norm is at most $b\beta^{|h|}$. Thus
$\tau_K(\rho)\to(1+\rho)/(1-\rho)$ and the variance per time step converges
to the sum of lag covariances. Dominated convergence, or summing the nonnegative
degree and lag terms, proves \eqref{eq:momentum-infinite}. The comparator limit
is $\sigma_\beta^2(1+\beta)/(1-\beta)=\sigma_g^2$.
For $\ell>1$ write
\begin{equation}
 q_\ell(\beta)=\frac{1-\beta}{1+\beta}
 \frac{1+\beta^\ell}{1-\beta^\ell}.
\end{equation}
Set $\beta=e^{-x}$, $x>0$. Then
$q_\ell=\tanh(x/2)/\tanh(\ell x/2)$ and
\begin{equation}
 \frac{d}{dx}\log q_\ell
 =\frac1{\sinh x}-\frac{\ell}{\sinh(\ell x)}>0,
\end{equation}
because $\sinh(\ell x)>\ell\sinh x$ for $\ell>1,x>0$. Therefore $q_\ell$
strictly decreases with $\beta$; $q_\ell(0)=1$ by continuity. Positive residual
mass makes $\mathcal R_\beta$ strictly decreasing on $[0,1)$.
For odd $\ell\ge3$, $q_\ell\le q_3$ and
$q_3=(1-\beta+\beta^2)/(1+\beta+\beta^2)$.
Finally $q_\ell\to1/\ell$ as $\beta\uparrow1$, and $0<q_\ell\le1$.
Dominated convergence with \eqref{eq:hermite-weights} gives
\begin{equation}
 \lim_{\beta\uparrow1}\mathcal R_\beta
 =1+\sum_{\substack{\ell\ge3\\\ell\text{ odd}}}\frac{w_\ell}{\ell c^2}
 \le1+(\mathcal R_0-1)/3.
\end{equation}
This proves the corollary. \hfill$\square$
\subsection{Scalar closed form}\label{app:momentum-scalar}
For scalar standard Gaussians with correlation $\rho$,
$\E[\operatorname{sign}(Z)\operatorname{sign}(Z')]=(2/\pi)\arcsin\rho$.
For completeness, write $(Z,Z')=(A,\rho A+\sqrt{1-\rho^2}D)$ with independent
standard $A,D$. Their polar angle is uniform; the angular measure of opposite
signs is $2\arccos\rho$ out of $2\pi$. Thus the expectation is
$1-2\arccos\rho/\pi=(2/\pi)\arcsin\rho$.
Using $c^2=2/\pi$ gives
\begin{align}
 \mathcal R_{\beta,K}
 &=\frac{\pi/2+2\sum_{h=1}^{K-1}(1-h/K)\arcsin(\beta^h)}{\tau_K(\beta)},\\
 \mathcal R_\beta
 &=\frac{1-\beta}{1+\beta}
 \left[\frac\pi2+2\sum_{h\ge1}\arcsin(\beta^h)\right].
\end{align}
For $\beta=e^{-x}\uparrow1$, monotone Riemann sums for the continuous integrable
function $t\mapsto\arcsin(e^{-t})$ imply
\begin{equation}
 \lim_{\beta\uparrow1}\mathcal R_\beta
 =\int_0^\infty\arcsin(e^{-t})\,dt
 =\int_0^1\frac{\arcsin u}{u}\,du=\frac\pi2\log2.
\end{equation}
For the last equality substitute $u=\sin\theta$ and integrate by parts. If
$I=\int_0^{\pi/2}\log\sin\theta\,d\theta$, reflection gives the same integral
with cosine, so $2I=I-(\pi/2)\log2$ by the double-angle formula; hence
$I=-(\pi/2)\log2$. This also verifies the stated constant without a numerical
fit. At $\beta=0.95$, the convergent series gives $\mathcal R_\beta\approx1.09197$.

\subsection{Window length and the order of limits}\label{app:window-proof}
We prove the window claims discussed in Section~\ref{sec:momentum}
and a quantitative bound. Write $\tau_K(\rho)=\sum_{h\ge0}a_{K,h}\rho^h$ with $a_{K,0}=1$ and
$a_{K,h}=2(1-h/K)_+$ for $h\ge1$.

\emph{(i) Monotonicity in $K$.} Let $K<K'$ and $h<h'$. Then
$a_{K,h}a_{K',h'}\ge a_{K,h'}a_{K',h}$: after removing the factor
$2-\boldsymbol1_{h=0}$, which depends only on $h$, and the factor $1/K$,
which depends only on $K$, the claim reads
$(K-h)_+(K'-h')_+\ge(K-h')_+(K'-h)_+$; if $h'\ge K$ the right side vanishes,
and otherwise both sides are positive with difference $(h'-h)(K'-K)\ge0$.
For $0\le\rho<\beta<1$, pairing the terms $(h,h')$ and $(h',h)$ gives
\begin{equation}
\begin{aligned}
 \tau_K(\rho)\tau_{K'}(\beta)-\tau_K(\beta)\tau_{K'}(\rho)
 &=\sum_{h<h'}\bigl(a_{K,h}a_{K',h'}-a_{K,h'}a_{K',h}\bigr)
 \rho^h\beta^h\bigl(\beta^{h'-h}-\rho^{h'-h}\bigr)\\
 &\ge 2\Bigl(\frac1K-\frac1{K'}\Bigr)(\beta-\rho)>0,
\end{aligned}
\end{equation}
the lower bound being the $(h,h')=(0,1)$ term. Hence
$\tau_{K'}(\rho)/\tau_{K'}(\beta)<\tau_K(\rho)/\tau_K(\beta)$. Taking
$\rho=\beta^\ell$ in \eqref{eq:momentum-finite} and using the positive
residual mass shows that $\mathcal R_{\beta,K}$ strictly decreases in $K$;
its limit is $\mathcal R_\beta$ by Corollary~\ref{cor:momentum-longrun}.
For the quantitative bound, the identity
$\sum_{h=1}^{K-1}(K-h)\rho^h=[K\rho(1-\rho)-\rho(1-\rho^K)]/(1-\rho)^2$ gives
\begin{equation}
 \tau_K(\rho)=\frac{1+\rho}{1-\rho}\bigl(1-\delta_K(\rho)\bigr),\qquad
 \delta_K(\rho)=\frac{2\rho(1-\rho^K)}{K(1-\rho^2)}
 =\frac{2}{K(1+\rho)}\sum_{j=1}^K\rho^j\in\Bigl[0,\frac{2\rho}{1+\rho}\Bigr].
 \label{eq:tau-closed}
\end{equation}
Therefore $\tau_K(\beta^\ell)/\tau_K(\beta)
=q_\ell(\beta)\,(1-\delta_K(\beta^\ell))/(1-\delta_K(\beta))
\le q_\ell(\beta)/(1-\delta_K(\beta))$, and summing over $\ell$ with
weights $w_\ell/c^2$ gives
$\mathcal R_{\beta,K}-1\le(\mathcal R_\beta-1)/(1-\delta_K(\beta))$.

\emph{(ii) Fixed window, $\beta\uparrow1$.} For fixed $K$, $\tau_K$ is
continuous on $[0,1]$ with $\tau_K(1)=K$, so
$\tau_K(\beta^\ell)/\tau_K(\beta)\to1$ for every $\ell$. The ratios are at
most one, and dominated convergence gives
$\mathcal R_{\beta,K}\to1+\sum_{\ell\ge3}w_\ell/c^2=\mathcal R_0$. Since
$\mathcal R_{0,K}=\mathcal R_0$ while $\mathcal R_{\beta,K}<\mathcal R_0$ for
$0<\beta<1$ and $K\ge2$, the map $\beta\mapsto\mathcal R_{\beta,K}$ is not
monotone.
\subsection{Other filters and orthogonalization maps}\label{app:momentum-filter}
The reference implementation \citep{jordan2024muon} uses the
Nesterov buffer
\begin{equation}
 B_k=\beta B_{k-1}+(1-\beta)G_k,\qquad M_k=(1-\beta)G_k+\beta B_k,
 \label{eq:nesterov-convention}
\end{equation}
and applies to $M_k$ the NS5 map: $X_0=G/\|G\|_F$ ($X_0=0$ if $G=0$),
followed by five iterations of
\begin{equation}
 X_{j+1}=3.4445X_j-4.7750(X_jX_j^\top)X_j+2.0315(X_jX_j^\top)^2X_j,
 \label{eq:ns5-map}
\end{equation}
with tall inputs transposed.
\begin{proposition}[Filters and orthogonalization maps]\label{prop:momentum-filter}
Let $T:\R^{m\times n}\to\R^{m\times n}$ be measurable, odd, positively
scale-invariant, and equivariant under independent orthogonal row and column
transformations, with $\E\|T(Z)\|_F^2<\infty$ and
$c_T=\E\langle T(Z),Z\rangle/d\ne0$; set $b_T=\E\|T(Z)\|_F^2/d$ and
$\mathcal R_{T,0}=b_T/c_T^2$. Let $a_j\ge0$ with $\sum_{j\ge0}a_j=1$ and
\begin{equation}
 M_k=Y+\sigma_g\sum_{j\ge0}a_jZ_{k-j},\qquad
 v=\sum_{j\ge0}a_j^2,\qquad
 \rho_h=\frac1v\sum_{j\ge0}a_ja_{j+h}.
 \label{eq:general-filter}
\end{equation}
At $Y=0$ the streams $U_k=\sigma_g\sqrt v\,T(M_k)/c_T$ and $V_k=M_k$ both have mean
response $I_d$, and their accumulated covariances satisfy
\eqref{eq:momentum-finite-cov} with
\begin{equation}
 \mathcal R_{T,K}=1+\sum_{\substack{\ell\ge3\\\ell\text{ odd}}}
 \frac{w_{T,\ell}}{c_T^2}\frac{\Lambda_{\ell,K}}{\Lambda_{1,K}},\qquad
 \Lambda_{\ell,K}=1+2\sum_{h=1}^{K-1}(1-h/K)\rho_h^\ell,
 \label{eq:general-filter-ratio}
\end{equation}
where $w_{T,\ell}\ge0$, $w_{T,1}=c_T^2$, and $\sum_{\ell\text{ odd}}w_{T,\ell}=b_T$.
Hence $1<\mathcal R_{T,K}\le1+(\mathcal R_{T,0}-1)\Lambda_{3,K}/\Lambda_{1,K}\le\mathcal R_{T,0}$,
and the long-window limit replaces $\Lambda_{\ell,K}$ by
$\Lambda_\ell=1+2\sum_{h\ge1}\rho_h^\ell$ (with $\Lambda_1=1/v$). The Nesterov
filter \eqref{eq:nesterov-convention} has $a_0=1-\beta^2$,
$a_j=(1-\beta)\beta^{j+1}$ ($j\ge1$),
\[
 v=\frac{1-\beta}{1+\beta}(1+2\beta-2\beta^3),\qquad
 \rho_h=\zeta_\beta\beta^h,\qquad
 \zeta_\beta=\frac{\beta+\beta^2-\beta^3}{1+2\beta-2\beta^3},
\]
and NS5 satisfies the hypotheses with $c_T>0$.
\end{proposition}
\begin{proof}
The series defining $M_k$ exists in $L^2$ since $\sum a_j^2\le1$. Its
standardized process is jointly Gaussian with lag correlations $\rho_h$.
Repeat the proof of Appendix~\ref{app:momentum-proof} with $f=\vecop T(Z)$ to obtain the stated
weights and covariance. Finite Gaussian second moment suffices for the score
derivative: the shifted Gaussian likelihood and its derivatives lie in $L^2$
on every bounded set of shifts, so Cauchy--Schwarz justifies differentiating
their pairing with $f$. Homogeneity gives the slope $c_T/(\sigma_g\sqrt v)$.
Nonnegative weights give $0\le\rho_h\le1$ and
\[
 1+2\sum_{h\ge1}\rho_h
 =\frac{\sum_j a_j^2+2\sum_{j<k}a_ja_k}{v}=1/v<\infty.
\]
This proves summability, the finite and infinite bounds, and the comparator's
long-run covariance $\sigma_g^2I_d$. The higher-degree mass is positive, so in fact
$\mathcal R_{T,K}>1$: if $T(Z)=c_TZ$ almost surely, then $T(w)=c_Tw$ for
Lebesgue-a.e.\ $w$ (the Gaussian law and Lebesgue measure share null sets),
so also $T(2Z)=2c_TZ$ almost surely, while positive scale invariance gives
$T(2Z)=T(Z)=c_TZ$ almost surely, contradicting
$c_T\ne0$. A strict reduction below $\mathcal R_{T,0}$ requires a lag
$h<K$ with $0<\rho_h<1$.
\end{proof}

\paragraph{Nesterov weights.}
Writing $M_k=(1-\beta)G_k+\beta B_k$ with $B_k=(1-\beta)\sum_{j\ge0}\beta^jG_{k-j}$
gives $a_0=(1-\beta)+\beta(1-\beta)=1-\beta^2$ and $a_j=(1-\beta)\beta^{j+1}$ for
$j\ge1$. Summing the geometric series gives the stated $v$ and
$\sum_ja_ja_{j+h}=v\,\zeta_\beta\beta^h$ for $h\ge1$, hence
$\rho_h=\zeta_\beta\beta^h$.

\paragraph{Finite Newton--Schulz maps.}
The NS5 map \eqref{eq:ns5-map} is a quintic polynomial approximation used in the reference Muon code
\citep{jordan2024muon}; it is not the exact polar map. Transposing tall matrices
before evaluation and transposing back is algebraically equivalent, by
associativity. Each iteration preserves oddness and orthogonal equivariance;
normalization gives positive scale invariance. A finite polynomial composition
on the unit Frobenius sphere is bounded, so the moment condition holds.
Moreover $3.4445-4.7750t^2+2.0315t^4$ is strictly positive for all real $t$,
since its minimum as a quadratic in $t^2$ is bounded below by
$3.4445-4.7750^2/(4\cdot2.0315)>0$. Each nonzero singular value therefore
remains positive. It follows that $\langle T(G),G\rangle>0$ for $G\ne0$,
and hence $c_T>0$.
\section{Proofs for Section~\ref{sec:local-consequences}}\label{app:switch-proof}
\subsection{Proof of Theorem~\ref{thm:switch-surrogate}}
\paragraph{The common state matrix.}
Write $\alpha=1-\beta$ and $J=J(y)$. Before centering, the state is
$(x_k,\widetilde b_{k-1})$. Substitution of \eqref{eq:feedback-buffer} into
\eqref{eq:local-linear-step} gives the affine recursion with constant
$a_0=(-\eta_M\mu(y),0)$ and matrices
\[
 A=\begin{pmatrix}I_d-\eta_M\alpha JH&-\eta_M\beta J\\
                    \alpha H&\beta I_d\end{pmatrix},\quad
 E_z=\begin{pmatrix}-\eta_M\alpha sJ\\\alpha sI_d\end{pmatrix},\quad
 E_e=\begin{pmatrix}-\eta_MI_d\\0\end{pmatrix}.
\]
Schur stability makes $I-A$ invertible. Both forcing terms are centered, so
both recursions have mean $(I-A)^{-1}a_0$. Subtracting it yields the theorem's
state $q_k$ and leaves $A,E_z,E_e$ unchanged. For the quadratic
$F_{\rm loc}(x)=F_0+y^\top x+x^\top Hx/2$, the expected loss at common
mean $\bar x$ is $F_{\rm loc}(\bar x)+\tr(H\Cov(x))/2$. Thus a possible
nonzero-signal bias cancels from the difference in \eqref{eq:switch-risk}.

\paragraph{Orthogonality and stationary covariance.}
For every pair $j,k$, the vectors $(n_k,z_j)$ are jointly Gaussian. The
conditional mean $\E[z_j\mid n_k]$ is therefore linear in $n_k$.
Theorem~\ref{thm:response} gives $\E e_k=0$ and $\E[e_kn_k^\top]=0$, hence
$\E[e_kz_j^\top]=\E[e_k\,\E(z_j\mid n_k)^\top]=0$. This orthogonality at all
lags does not require $e_k$ to be independent of $z_j$.

Since $A$ is Schur stable, the stationary solutions are the $L^2$ limits
\[
 q_k^L=\sum_{h\ge0}A^hE_zz_{k-1-h},\qquad
 q_k^e=\sum_{h\ge0}A^hE_ee_{k-1-h},\qquad q_k^M=q_k^L+q_k^e.
\]
The sums converge in $L^2$ because $\sum_h\|A^h\|<\infty$ and the inputs have
finite second moments. Every cross-covariance between $q_k^L$ and $q_k^e$
vanishes by the orthogonality at all lags, first for finite sums and then in
the $L^2$ limit. Hence $S_M=\Cov(q_k^L)+\Cov(q_k^e)=S_L+S_e$ with
$S_e\succeq0$. Applying the congruence $\Pi(\cdot)\Pi^\top$ and taking the
trace against $H\succeq0$ proves \eqref{eq:switch-risk}.

\paragraph{Strictness for polar updates.}
In the stated model $H\succ0$ and $\eta_M>0$. Write the stationary residual-only
state as $q_k^e=(x_k^e,b_{k-1}^e)$. If the loss gap vanishes, then
$\E[(x_k^e)^\top Hx_k^e]=0$, hence $x_k^e=0$ almost surely at every integer
time. Its buffer equation becomes $b_k^e=\beta b_{k-1}^e$. Stationarity and
$\beta<1$ imply $\E\|b_k^e\|^2=0$, and its parameter equation then gives
$\eta_Me_k=0$ almost surely. Conversely, a zero residual gives $S_e=0$.
Thus a nonzero residual necessarily contributes strictly positive loss.

For exact polar $T$, the residual is nonzero at every fixed $y$. Otherwise
$T(y+n)=\mu(y)+J(y)n$ almost surely for a nondegenerate Gaussian $n$.
The polar map is bounded, whereas any nonzero linear image of $n$ is
unbounded, so this would force $J(y)=0$ and make $T$ constant almost surely.
That is impossible: a full-support Gaussian assigns positive probability
to neighborhoods of a full-rank matrix and its negative, where the polar
map is continuous and takes distinct values. This proves strictness without
requiring a scalar Jacobian or a vanishing step size.

\paragraph{Convergence after a fixed-time switch.}
For a switch at a fixed finite time $t$, couple the switched state $\widetilde q$ with the
stationary linear solution $q^L$ through the same future innovations. Their
difference after $n$ steps is
$\widetilde q_{t+n}-q^L_{t+n}=A^n(\widetilde q_t-q^L_t)$. It tends to zero in
$L^2$ for any square-integrable retained state, so its mean and covariance
converge to those of the response-matched comparator. Convergence of the expected
quadratic loss follows as well: for its parameter coordinates, the absolute
difference of expected losses is at most
\[
 \bigl(\|y\|+\tfrac12\|H\|_{\op}
 [\|\widetilde x_{t+n}\|_{L^2}+\|x^L_{t+n}\|_{L^2}]\bigr)
 \|\widetilde x_{t+n}-x^L_{t+n}\|_{L^2}\longrightarrow0,
\]
where $\|X\|_{L^2}=(\E\|X\|^2)^{1/2}$ and both trajectories include their
common affine mean. The norms in brackets remain bounded by the same
coupling. The retained state may come from nonlinear Muon; this establishes
the destination after switching, but does not compare it with the stationary
loss of continued nonlinear Muon. No immediate loss ordering follows.
\hfill$\square$

\subsection{Proof of Theorem~\ref{thm:momentum-finite-step}}\label{app:momentum-finite-step}
Write $\eta=\eta_L$, $z_k=\vecop Z_k$, and $\epsilon_k=(\sigma_\beta/c)e_k$, so
that the surrogate is $x_{k+1}=x_k-\eta(\widetilde b_k+\epsilon_k)$ and matched
momentum SGD is the same recursion without $\epsilon_k$.

\paragraph{The residual process.}
At zero signal, $(\sigma_\beta/c)P(n_k)=(\sigma_\beta/c)f(n_k/\sigma_\beta)
=n_k+\epsilon_k$ with $\epsilon_k=(\sigma_\beta/c)\sum_{\ell\ge3}f_\ell(n_k/\sigma_\beta)$,
in the notation of Appendix~\ref{app:momentum-proof}, because $f_1(x)=cx$.
The vectors $n_k/\sigma_\beta$ and $z_j$ are jointly Gaussian with
cross-covariance a multiple of $I_d$, so \eqref{eq:chaos-cross-cov} with $j=1$
gives $\E[\epsilon_kz_j^\top]=0$ for all $k,j$, and with $\rho=\beta^{|h|}$ it
gives $\Cov(\epsilon_k,\epsilon_{k+h})=(\sigma_\beta^2/c^2)\sum_{\ell\ge3}w_\ell\beta^{\ell|h|}I_d$.
Distinct degrees are uncorrelated at all lags.

\paragraph{Mode decomposition.}
Let $H=U\diag(\lambda_i)U^\top$. All input covariances are multiples of
$I_d$, so in the coordinates $U^\top x$ the recursion decouples into $d$
scalar recursions with curvature $\lambda=\lambda_i$,
\[
 x_{k+1}=x_k-\eta(b_k+\epsilon_k),\qquad
 b_k=\beta b_{k-1}+(1-\beta)(\lambda x_k+\sigma_g\nu_k),
\]
driven by a white unit-variance $\nu_k$ and a residual $\epsilon_k$ that is
uncorrelated with $\nu$ at all lags. Eliminating $b$ gives the
characteristic polynomial $\zeta^2-A\zeta+\beta$ with
$A=1+\beta-\eta\lambda(1-\beta)$, whose roots lie in the open unit disk if and
only if $|A|<1+\beta$, that is, $0<\eta\lambda<2/v_\beta$. On this interval the
recursion has a unique causal stationary solution in $L^2$ for every
stationary finite-variance input, and its covariance is the limit from any
square-integrable initial state. Since the inputs are uncorrelated at all
lags, the output variances of the Gaussian buffer component and of each
residual degree add.

\paragraph{Response to a single covariance component.}
For one eigenmode, the response to a scalar forcing $u_k$ of covariance
$\E u_ku_{k+h}=q^{|h|}$ satisfies
\begin{equation}
 y_{k+1}=Ay_k-\beta y_{k-1}-\eta u_k+\eta\beta u_{k-1}.
 \label{eq:single-chaos-recursion}
\end{equation}
Only second moments enter the calculation; thus $u_k$ may be replaced by a
unit-variance Gaussian AR(1) process with parameter $q$ even when the actual
Hermite component is not Gaussian. Let
$C=\E[y_{k+1}u_k]$, $V=\E y_k^2$, and $G=\E[y_{k+1}y_k]$.
Causality and the AR covariance imply
$\E[y_ku_k]=qC$ and $\E[y_{k-1}u_k]=q^2C$. Multiplying
\eqref{eq:single-chaos-recursion} by $u_k$ and $y_k$ gives
\begin{equation}
 C=-\frac{\eta(1-\beta q)}{1-Aq+\beta q^2},\qquad
 (1+\beta)G=AV+\eta(\beta-q)C.
\end{equation}
Squaring the recursion, taking expectations, and eliminating $G$ yields
\begin{equation}
 \left(1-\beta^2-\frac{1-\beta}{1+\beta}A^2\right)V
 =\eta^2(1+\beta^2-2\beta q)
 +2\eta(\beta-q)\left(\frac A{1+\beta}-\beta q\right)C.
 \label{eq:single-chaos-variance}
\end{equation}
Write $t=\eta\lambda$ and $D_\beta(q,t)=1-Aq+\beta q^2$.
The variance at $q=\beta$ is
\begin{equation}
 V(\beta)=\frac{\eta^2(1+\beta)^2}
 {t(1-\beta)\{2(1+\beta)-t(1-\beta)\}}.
\end{equation}
Substitution of $C$ in \eqref{eq:single-chaos-variance}, followed by division
by $V(\beta)$, gives exactly $V(q)/V(\beta)=L_\beta(q,t)$.
The Gaussian buffer component has variance $\sigma_\beta^2=\sigma_g^2v_\beta$,
so its output variance is
$\sigma_\beta^2V(\beta)=\eta\sigma_g^2/(2\lambda-\eta v_\beta\lambda^2)$.
The degree-$\ell$ residual has variance $\sigma_\beta^2w_\ell/c^2$ and lag
parameter $q=\beta^\ell$. Distinct degrees, and the residual and the input
Gaussian component, are uncorrelated at all lags. Their output variances add.
The same argument across coordinates gives \eqref{eq:finite-step-covariance}.

\paragraph{Monotonicity and crossing.}
For $0\le q<\beta<1$, $D_\beta(q,t)>0$ and
\begin{equation}
 \partial_t L_\beta(q,t)=
 \frac{2(\beta-q)(1-\beta q)^2(1+\beta q)}
 {(1+\beta)^2D_\beta(q,t)^2}>0.
 \label{eq:finite-step-monotonicity}
\end{equation}
At the two ends of the stability interval,
\begin{equation}
 L_\beta(q,0)=\frac{1-\beta}{1+\beta}\frac{1+q}{1-q},\qquad
 \lim_{t\uparrow2/v_\beta}L_\beta(q,t)
 =\frac{1+\beta}{1-\beta}\frac{1-q}{1+q}.
\end{equation}
These values are positive and bounded above by $(1+\beta)/(1-\beta)$,
uniformly in $q\in[0,\beta]$. The summability of the Hermite weights therefore
justifies all series, continuity, and endpoint limits. Positive residual mass
gives strict monotonicity of $\mathcal R^{\rm lin}_\beta$.
The degree-$\ell$ factor crosses one at
\begin{equation}
 t_\ell=\frac{(1+\beta)(1-\beta^{\ell+1})}{1+\beta^{\ell+2}}.
 \label{eq:degree-threshold}
\end{equation}
This increases with $\ell$, with minimum $t_3$ and limit $1+\beta$.
Thus $\mathcal R^{\rm lin}_\beta(t_3)\le\mathcal R_0$, whereas
$\mathcal R^{\rm lin}_\beta(1+\beta)>\mathcal R_0$. Both points are strictly inside
the stability interval. The intermediate value theorem and strict
monotonicity give the unique crossing and its bounds. Finally
$L_\beta(\beta^\ell,0)=q_\ell(\beta)$, so
$\mathcal R^{\rm lin}_\beta(0)=\mathcal R_\beta$.

\paragraph{Loss ratio.}
By \eqref{eq:finite-step-covariance},
$\tr(H\Sigma_M)=\sum_i\omega_i\mathcal R^{\rm lin}_\beta(\eta\lambda_i)$ and
$\tr(H\Sigma_L)=\sum_i\omega_i$ with positive weights
$\omega_i=\eta\sigma_g^2/(2-\eta v_\beta\lambda_i)$, which gives the weighted-average
statement; monotonicity gives the interval, $\mathcal R^{\rm lin}_\beta\ge\mathcal R_\beta>1$
gives the strict lower bound, and continuity at $t=0$ gives the small-step
limit.

\paragraph{The case $\beta=0$.}
Then $b_k=\lambda x_k+\sigma_g\nu_k$, and both $\nu_k$ and $\epsilon_k$ are white, with
$\Var(\epsilon_k)=\sigma_g^2(\mathcal R_0-1)$ by \eqref{eq:one-step-ratio}. The scalar
recursion $x_{k+1}=(1-\eta\lambda)x_k-\eta(\sigma_g\nu_k+\epsilon_k)$ is stable for
$0<\eta\lambda<2$, and its stationary variance is
$\eta^2[\sigma_g^2+\Var(\epsilon_k)]/(1-(1-\eta\lambda)^2)$, which is $\mathcal R_0$ times the
variance without $\epsilon_k$ in every mode. \hfill$\square$

\subsection{Scalar response is an additional approximation}\label{app:scalarity}
Let $\xi$ be centered with covariance $\Sigma_B\succ0$, and fix a square
response operator $J$. A scalar comparator $a\xi$ approximates $J\xi$ in
mean square according to
\[
 \E\|(J-aI_d)\xi\|^2
 =\tr(J\Sigma_BJ^\top)-2a\tr(J\Sigma_B)+a^2\tr\Sigma_B.
\]
The minimizing scalar is
\[
 a_* =\frac{\tr(J\Sigma_B)}{\tr\Sigma_B},\qquad
 \min_a\E\|(J-aI_d)\xi\|^2
 =\tr(J\Sigma_BJ^\top)
   -\frac{\tr(J\Sigma_B)^2}{\tr\Sigma_B}.
\]
These identities follow by expanding the square and differentiating the
resulting quadratic in $a$. In contrast,
$a_{\rm eff}=\sqrt{\tr(J\Sigma_BJ^\top)/\tr\Sigma_B}$ matches only the
RMS size of $J\xi$, not its direction or sign. Zero approximation error
requires $J=a_*I_d$ on the support of the noise (on all of $\R^d$ here).
For general $J$, $a_*$ can even be nonpositive, so it need not define an
admissible positive SGD rate. A small noise-weighted approximation error
also leaves response in weak-noise curvature directions unchecked.
Finally, the exact local comparator is affine: replacing it by $a_*u$
also drops the intercept $\mu(y)-J(y)y$. Neither diagnostic by itself
certifies an optimizer switch.
\section{Experimental methods}\label{app:main-methods}
\paragraph{Maps and filters.}
$T$ is the exact polar factor (float64 SVD) or NS5 \eqref{eq:ns5-map}. EMA and
Nesterov buffers follow \eqref{eq:muon-momentum-def} and
\eqref{eq:nesterov-convention}, with weights and noise variance $\sigma_g^2v$ as in
Proposition~\ref{prop:momentum-filter}. With
$c_T=\E\langle T(Z),Z\rangle/d$, the matched Muon step for a linear rate $\eta$
is $\eta\sigma_g\sqrt v/c_T$. Synthetic runs use float64 and zero normalization
epsilon. Table~\ref{tab:momentum-constants} lists the Gaussian constants,
Table~\ref{tab:momentum-chains} the full chain grid, and
Table~\ref{tab:network-ratios} the full transformer grid.

\paragraph{Mean polar update (Figure~\ref{fig:exp-mean}).}
For $\|Y\|_F=0.1\sigma$ and a random direction $Y$, the mean update is estimated
from antithetic pairs $\tfrac12[P(Y+\sigma Z)+P(Y-\sigma Z)]$, $16{,}000$ pairs per
slope and $40{,}000$ for the entrywise panel, and projected as
$\langle\widehat\Psi_\sigma(Y),Y\rangle/\|Y\|_F^2$. The prediction $c_{N,N}/\sigma$
uses $10{,}000$ Gaussian matrices.

\paragraph{Hermite weights (Figure~\ref{fig:exp-momentum}a,b).}
For each shape and map, 16 batches draw Gaussian pairs $X$ and
$X'_\theta=\cos\theta\,X+\sin\theta\,Z$ on 97 angles in $[0,\pi/2]$ and
record the residual two-point function $\widehat D(\rho)$, $\rho=\cos\theta$,
with the degree-one part removed as a control variate; its expectation is
$\sum_{\ell\ge3}w_\ell\rho^\ell$. Nonnegative least squares on
$\{\rho^\ell\}_{\ell=3,5,\ldots,401}$ gives the weights, which enter
Corollary~\ref{cor:zero-signal} and its filter and map extensions. Batch
sizes range from $200{,}000$ ($1\times1$) to $500$ ($64\times64$); the
truncation error, checked against the exact scalar curve, is below
$3\times10^{-4}$.

\paragraph{Quadratic chains (Figure~\ref{fig:exp-momentum}c).}
$F(x)=\tfrac12x^\top Hx$ with dense $H$, eigenvalues geometrically spaced
in $[0.5,2]$, and gradient-noise scale $\sigma_g=0.2$. Muon and SGD share the filter and
each innovation within a chain pair; buffers and iterates start at zero.
Each configuration uses 32 pairs, discards $\lceil16/\eta\rceil$ steps, and
averages the loss over $\lceil150/\eta\rceil$ steps for
$\eta\in\{10^{-3},2.5\times10^{-4},6.25\times10^{-5}\}$. The grid has
scalar, $4\times4$, and $4\times16$ shapes with exact polar EMA at
$\beta\in\{0,0.5,0.9,0.95\}$; the matrix shapes add NS5 and the Nesterov
filter, in all combinations, at $\beta\in\{0.9,0.95\}$. Ratios of mean
excess losses carry delta-method intervals over chain pairs.

\paragraph{Step-size sweep (Table~\ref{tab:exp-steps}).}
$F(E)=\tfrac12\tr(E^\top PEQ)$ on $16\times16$ matrices, with $P,Q$ random
rotations of geometrically spaced eigenvalues so that the Hessian
eigenvalues lie in $[0.09,1]$, and gradient noise $\sigma_gZ_k$ with $\sigma_g=0.2$. Muon uses EMA
$\beta=0.9$ with exact polar or NS5 and step $\eta_L\sigma_\beta/c_T$;
momentum SGD uses the same buffer and step $\eta_L$. Sixty-four paired
chains start at zero, discard $8/(\eta_L\lambda_{\min})$ steps, and average
the loss over $40/(\eta_L\lambda_{\min})$ steps. The surrogate curve is
Theorem~\ref{thm:momentum-finite-step} (for exact polar) or its analogue with the
NS5 weights, evaluated with the estimated Hermite weights.

\paragraph{Frozen transformer streams (Figure~\ref{fig:exp-network}).}
A six-layer pre-LayerNorm GPT ($d_{\rm model}=256$, four heads, context
256, no biases or dropout) is trained on Tiny Shakespeare for 400 steps
with the reference Muon update (Nesterov $0.95$, bfloat16 NS5, learning
rate $0.02$ decayed linearly over the last 200 steps) on its hidden
matrices and AdamW elsewhere; the validation loss is $1.518$ at step 400
and $1.893$ at step 100. At a frozen checkpoint, $N=8192$ minibatch
gradients (16 sequences each) are stored for q, v, o, fc, and proj of
block 3 and fc of block 0. A stream with effective batch 64 averages four
gradients drawn with replacement; conditional on the pool it is i.i.d.\
with mean $\bar Y$. Each of six filters ($\beta=0$; EMA $0.5,0.9,0.95$;
Nesterov $0.9,0.95$) runs 96 chains (48 for exact polar) for 200 burn-in
and 1024 recorded steps, accumulating $T(M_k)$ and $M_k$ over
$K\in\{1,4,16,64,256\}$. Controls pass isotropic Gaussian noise, Gaussian
noise with the pool covariance, and both with $\bar Y$ added, through the
same pipeline. The indicator $\delta_\beta$ in \eqref{eq:delta-beta} is
computed at EMA $\beta=0.95$ from the unbiased split-pool estimate of
$\|Y\|_F^2$ and the pool noise covariance; at step 400 it is $1.00$, $1.74$,
$1.22$, $1.30$, $1.49$, and $1.41$ for q, v, o, fc, proj of block 3 and fc of
block 0, and at step 100 it is $10.7$, $9.7$, and $8.4$ for q and fc of block
3 and fc of block 0.

The denominator of \eqref{eq:network-main-ratio} is
$a_{\rm eff}^2\tr\Cov(\sum_kM_k)$ with $a_{\rm eff}^2=\tr(JSJ^\top)/\tr S$
and $S=\Cov(M_k)$; for a scalar linear filter of i.i.d.\ innovations,
$\Cov(\sum_kM_k)$ is proportional to $S$, so this equals
$\tr\Cov(\sum_kJM_k)$ without any isotropy assumption. $a_{\rm eff}^2$ is
estimated by central differences along centered buffer directions with two
independent buffers, common random numbers, perturbation scales $0.03$ and
$0.1$, and Richardson extrapolation; the two raw scales change
$\mathcal R^{\rm lin}_{256}$ by at most $0.01$. Intervals are
delete-one-group jackknife intervals over 16 groups of chains and condition
on the stored pool.

\begin{table}[t]
\centering
\footnotesize
\setlength{\tabcolsep}{2.5pt}
\resizebox{\linewidth}{!}{%
\begin{tabular}{llccccccc}
\toprule
Shape & Map & $\mathcal R_0$ & EMA $0.9$ & EMA $0.95$ & bound $0.95$ & EMA $\beta\uparrow1$ & Nesterov $0.95$ & $w_3$ share \\
\midrule
$1\times1$ & polar & $1.571$ & $1.097\pm0.001$ & $1.091\pm0.001 (1.092)$ & $1.191$ & $1.088\pm0.001$ & $1.062\pm0.001$ & $0.29$ \\
$1\times16$ & polar & $1.032$ & $1.010\pm0.000$ & $1.010\pm0.000$ & $1.011$ & $1.010\pm0.000$ & $1.007\pm0.000$ & $0.88$ \\
$1\times16$ & NS5 & $1.031$ & $1.010\pm0.000$ & $1.010\pm0.000$ & $1.011$ & $1.010\pm0.000$ & $1.007\pm0.000$ & $0.88$ \\
$4\times4$ & polar & $1.471$ & $1.090\pm0.000$ & $1.086\pm0.000$ & $1.157$ & $1.084\pm0.000$ & $1.059\pm0.000$ & $0.36$ \\
$4\times4$ & NS5 & $1.441$ & $1.078\pm0.000$ & $1.074\pm0.000$ & $1.147$ & $1.072\pm0.000$ & $1.051\pm0.000$ & $0.32$ \\
$4\times16$ & polar & $1.086$ & $1.026\pm0.000$ & $1.026\pm0.000$ & $1.029$ & $1.026\pm0.000$ & $1.018\pm0.000$ & $0.80$ \\
$4\times16$ & NS5 & $1.137$ & $1.034\pm0.000$ & $1.034\pm0.000$ & $1.046$ & $1.033\pm0.000$ & $1.023\pm0.000$ & $0.51$ \\
$16\times16$ & polar & $1.413$ & $1.084\pm0.000$ & $1.081\pm0.000$ & $1.138$ & $1.080\pm0.000$ & $1.056\pm0.000$ & $0.40$ \\
$16\times16$ & NS5 & $1.340$ & $1.057\pm0.000$ & $1.054\pm0.000$ & $1.114$ & $1.053\pm0.000$ & $1.036\pm0.000$ & $0.27$ \\
$64\times64$ & polar & $1.394$ & $1.083\pm0.000$ & $1.080\pm0.000$ & $1.132$ & $1.079\pm0.000$ & $1.055\pm0.000$ & $0.40$ \\
$64\times64$ & NS5 & $1.509$ & $1.103\pm0.000$ & $1.100\pm0.000$ & $1.170$ & $1.099\pm0.000$ & $1.067\pm0.000$ & $0.38$ \\
\bottomrule
\end{tabular}}
\caption{Accumulated-noise ratios from estimated Hermite weights (95\% intervals over 16 batches). `bound' is the degree-3 bound of Corollary~\ref{cor:momentum-longrun} evaluated with the estimated $\mathcal R_0$; for NS5 it is the bound of Proposition~\ref{prop:momentum-filter}. The scalar exact value is in parentheses. The $w_3$ share is $w_3/\sum_{\ell\ge3}w_\ell$.}
\label{tab:momentum-constants}
\end{table}
{\small\setlength{\tabcolsep}{4pt}
\begin{longtable}{llllcccc}
\toprule
Shape & Map & Filter & $\beta$ & $\eta$ & Predicted & Observed & RMS $\|g\|_F/(s\sqrt v)$ \\
\midrule
\endhead
$1\times1$ & polar & EMA & $0.0$ & $0.001$ & $1.571$ & $1.594\pm0.062$ & $0.03$ \\
$1\times1$ & polar & EMA & $0.0$ & $0.00025$ & $1.571$ & $1.553\pm0.052$ & $0.01$ \\
$1\times1$ & polar & EMA & $0.0$ & $6.25e-05$ & $1.571$ & $1.557\pm0.041$ & $0.01$ \\
$1\times1$ & polar & EMA & $0.5$ & $0.001$ & $1.208$ & $1.206\pm0.035$ & $0.04$ \\
$1\times1$ & polar & EMA & $0.5$ & $0.00025$ & $1.208$ & $1.226\pm0.028$ & $0.02$ \\
$1\times1$ & polar & EMA & $0.5$ & $6.25e-05$ & $1.208$ & $1.217\pm0.033$ & $0.01$ \\
$1\times1$ & polar & EMA & $0.9$ & $0.001$ & $1.098$ & $1.091\pm0.017$ & $0.10$ \\
$1\times1$ & polar & EMA & $0.9$ & $0.00025$ & $1.098$ & $1.100\pm0.022$ & $0.05$ \\
$1\times1$ & polar & EMA & $0.9$ & $6.25e-05$ & $1.098$ & $1.100\pm0.021$ & $0.03$ \\
$1\times1$ & polar & EMA & $0.95$ & $0.001$ & $1.092$ & $1.091\pm0.019$ & $0.14$ \\
$1\times1$ & polar & EMA & $0.95$ & $0.00025$ & $1.092$ & $1.091\pm0.017$ & $0.07$ \\
$1\times1$ & polar & EMA & $0.95$ & $6.25e-05$ & $1.092$ & $1.103\pm0.021$ & $0.04$ \\
$4\times4$ & polar & EMA & $0.0$ & $0.001$ & $1.472$ & $1.478\pm0.011$ & $0.11$ \\
$4\times4$ & polar & EMA & $0.0$ & $0.00025$ & $1.472$ & $1.472\pm0.013$ & $0.06$ \\
$4\times4$ & polar & EMA & $0.0$ & $6.25e-05$ & $1.472$ & $1.458\pm0.011$ & $0.03$ \\
$4\times4$ & polar & EMA & $0.5$ & $0.001$ & $1.173$ & $1.170\pm0.007$ & $0.17$ \\
$4\times4$ & polar & EMA & $0.5$ & $0.00025$ & $1.173$ & $1.173\pm0.006$ & $0.09$ \\
$4\times4$ & polar & EMA & $0.5$ & $6.25e-05$ & $1.173$ & $1.174\pm0.007$ & $0.04$ \\
$4\times4$ & polar & EMA & $0.9$ & $0.001$ & $1.090$ & $1.089\pm0.005$ & $0.43$ \\
$4\times4$ & polar & EMA & $0.9$ & $0.00025$ & $1.090$ & $1.088\pm0.005$ & $0.21$ \\
$4\times4$ & polar & EMA & $0.9$ & $6.25e-05$ & $1.090$ & $1.090\pm0.005$ & $0.11$ \\
$4\times4$ & polar & EMA & $0.95$ & $0.001$ & $1.086$ & $1.082\pm0.005$ & $0.61$ \\
$4\times4$ & polar & EMA & $0.95$ & $0.00025$ & $1.086$ & $1.084\pm0.004$ & $0.30$ \\
$4\times4$ & polar & EMA & $0.95$ & $6.25e-05$ & $1.086$ & $1.084\pm0.005$ & $0.15$ \\
$4\times4$ & polar & Nesterov & $0.9$ & $0.001$ & $1.061$ & $1.065\pm0.004$ & $0.36$ \\
$4\times4$ & polar & Nesterov & $0.9$ & $0.00025$ & $1.061$ & $1.061\pm0.004$ & $0.18$ \\
$4\times4$ & polar & Nesterov & $0.9$ & $6.25e-05$ & $1.061$ & $1.063\pm0.004$ & $0.09$ \\
$4\times4$ & NS5 & EMA & $0.9$ & $0.001$ & $1.079$ & $1.079\pm0.005$ & $0.42$ \\
$4\times4$ & NS5 & EMA & $0.9$ & $0.00025$ & $1.079$ & $1.075\pm0.004$ & $0.21$ \\
$4\times4$ & NS5 & EMA & $0.9$ & $6.25e-05$ & $1.079$ & $1.081\pm0.004$ & $0.11$ \\
$4\times4$ & NS5 & Nesterov & $0.9$ & $0.001$ & $1.055$ & $1.057\pm0.004$ & $0.36$ \\
$4\times4$ & NS5 & Nesterov & $0.9$ & $0.00025$ & $1.055$ & $1.056\pm0.004$ & $0.18$ \\
$4\times4$ & NS5 & Nesterov & $0.9$ & $6.25e-05$ & $1.055$ & $1.057\pm0.003$ & $0.09$ \\
$4\times4$ & polar & Nesterov & $0.95$ & $0.001$ & $1.059$ & $1.060\pm0.005$ & $0.55$ \\
$4\times4$ & polar & Nesterov & $0.95$ & $0.00025$ & $1.059$ & $1.059\pm0.004$ & $0.28$ \\
$4\times4$ & polar & Nesterov & $0.95$ & $6.25e-05$ & $1.059$ & $1.058\pm0.004$ & $0.14$ \\
$4\times4$ & NS5 & EMA & $0.95$ & $0.001$ & $1.075$ & $1.070\pm0.004$ & $0.61$ \\
$4\times4$ & NS5 & EMA & $0.95$ & $0.00025$ & $1.075$ & $1.074\pm0.004$ & $0.30$ \\
$4\times4$ & NS5 & EMA & $0.95$ & $6.25e-05$ & $1.075$ & $1.074\pm0.005$ & $0.15$ \\
$4\times4$ & NS5 & Nesterov & $0.95$ & $0.001$ & $1.051$ & $1.054\pm0.003$ & $0.55$ \\
$4\times4$ & NS5 & Nesterov & $0.95$ & $0.00025$ & $1.051$ & $1.051\pm0.005$ & $0.28$ \\
$4\times4$ & NS5 & Nesterov & $0.95$ & $6.25e-05$ & $1.051$ & $1.054\pm0.003$ & $0.14$ \\
$4\times16$ & polar & EMA & $0.0$ & $0.001$ & $1.085$ & $1.085\pm0.002$ & $0.19$ \\
$4\times16$ & polar & EMA & $0.0$ & $0.00025$ & $1.085$ & $1.085\pm0.002$ & $0.10$ \\
$4\times16$ & polar & EMA & $0.0$ & $6.25e-05$ & $1.085$ & $1.087\pm0.003$ & $0.05$ \\
$4\times16$ & polar & EMA & $0.5$ & $0.001$ & $1.035$ & $1.035\pm0.001$ & $0.33$ \\
$4\times16$ & polar & EMA & $0.5$ & $0.00025$ & $1.035$ & $1.035\pm0.002$ & $0.16$ \\
$4\times16$ & polar & EMA & $0.5$ & $6.25e-05$ & $1.035$ & $1.036\pm0.002$ & $0.08$ \\
$4\times16$ & polar & EMA & $0.9$ & $0.001$ & $1.026$ & $1.027\pm0.001$ & $0.82$ \\
$4\times16$ & polar & EMA & $0.9$ & $0.00025$ & $1.026$ & $1.026\pm0.001$ & $0.41$ \\
$4\times16$ & polar & EMA & $0.9$ & $6.25e-05$ & $1.026$ & $1.027\pm0.001$ & $0.21$ \\
$4\times16$ & polar & EMA & $0.95$ & $0.001$ & $1.025$ & $1.025\pm0.001$ & $1.18$ \\
$4\times16$ & polar & EMA & $0.95$ & $0.00025$ & $1.025$ & $1.026\pm0.001$ & $0.59$ \\
$4\times16$ & polar & EMA & $0.95$ & $6.25e-05$ & $1.025$ & $1.025\pm0.001$ & $0.29$ \\
$4\times16$ & polar & Nesterov & $0.9$ & $0.001$ & $1.016$ & $1.017\pm0.001$ & $0.71$ \\
$4\times16$ & polar & Nesterov & $0.9$ & $0.00025$ & $1.016$ & $1.017\pm0.001$ & $0.35$ \\
$4\times16$ & polar & Nesterov & $0.9$ & $6.25e-05$ & $1.016$ & $1.016\pm0.001$ & $0.18$ \\
$4\times16$ & NS5 & EMA & $0.9$ & $0.001$ & $1.034$ & $1.034\pm0.001$ & $0.83$ \\
$4\times16$ & NS5 & EMA & $0.9$ & $0.00025$ & $1.034$ & $1.033\pm0.001$ & $0.41$ \\
$4\times16$ & NS5 & EMA & $0.9$ & $6.25e-05$ & $1.034$ & $1.035\pm0.001$ & $0.21$ \\
$4\times16$ & NS5 & Nesterov & $0.9$ & $0.001$ & $1.021$ & $1.021\pm0.001$ & $0.71$ \\
$4\times16$ & NS5 & Nesterov & $0.9$ & $0.00025$ & $1.021$ & $1.021\pm0.001$ & $0.35$ \\
$4\times16$ & NS5 & Nesterov & $0.9$ & $6.25e-05$ & $1.021$ & $1.022\pm0.001$ & $0.18$ \\
$4\times16$ & polar & Nesterov & $0.95$ & $0.001$ & $1.018$ & $1.017\pm0.001$ & $1.08$ \\
$4\times16$ & polar & Nesterov & $0.95$ & $0.00025$ & $1.018$ & $1.019\pm0.001$ & $0.54$ \\
$4\times16$ & polar & Nesterov & $0.95$ & $6.25e-05$ & $1.018$ & $1.018\pm0.001$ & $0.27$ \\
$4\times16$ & NS5 & EMA & $0.95$ & $0.001$ & $1.033$ & $1.032\pm0.001$ & $1.18$ \\
$4\times16$ & NS5 & EMA & $0.95$ & $0.00025$ & $1.033$ & $1.034\pm0.001$ & $0.59$ \\
$4\times16$ & NS5 & EMA & $0.95$ & $6.25e-05$ & $1.033$ & $1.034\pm0.001$ & $0.30$ \\
$4\times16$ & NS5 & Nesterov & $0.95$ & $0.001$ & $1.022$ & $1.022\pm0.001$ & $1.08$ \\
$4\times16$ & NS5 & Nesterov & $0.95$ & $0.00025$ & $1.022$ & $1.022\pm0.001$ & $0.54$ \\
$4\times16$ & NS5 & Nesterov & $0.95$ & $6.25e-05$ & $1.022$ & $1.023\pm0.001$ & $0.27$ \\
\bottomrule
\caption{All moving-chain results. Observed: ratio of mean excess risks of momentum Muon and momentum SGD with the same filter, with 95\% delta-method intervals over 32 paired chains, conditional on the calibration constant. Predicted: $\mathcal R_\beta$ from Table~\ref{tab:momentum-constants} (exact for the scalar case).}
\label{tab:momentum-chains}\\
\end{longtable}}

\begin{table}[h]
\centering
\footnotesize
\setlength{\tabcolsep}{2.5pt}
\resizebox{\linewidth}{!}{%
\begin{tabular}{llcccccccc}\toprule
layer & noise & $\beta=0$ & EMA .5 & EMA .9 & EMA .95 & Nest.\ .9 & Nest.\ .95 & red.\ EMA & red.\ Nest.\ \\\midrule
L3 q & iso. Gaussian & 1.506\,$\pm$\,0.002 & 1.196\,$\pm$\,0.001 & 1.119\,$\pm$\,0.001 & 1.119\,$\pm$\,0.001 & 1.079\,$\pm$\,0.001 & 1.083\,$\pm$\,0.002 & 0.76 & 0.84 \\
L3 q & iso. Gaussian $+\bar Y$ & 1.511\,$\pm$\,0.002 & 1.202\,$\pm$\,0.001 & 1.126\,$\pm$\,0.001 & 1.118\,$\pm$\,0.001 & 1.089\,$\pm$\,0.001 & 1.089\,$\pm$\,0.001 & 0.77 & 0.83 \\
L3 q & cov.\ Gaussian & 1.537\,$\pm$\,0.006 & 1.204\,$\pm$\,0.004 & 1.118\,$\pm$\,0.003 & 1.120\,$\pm$\,0.003 & 1.078\,$\pm$\,0.003 & 1.080\,$\pm$\,0.003 & 0.78 & 0.85 \\
L3 q & cov.\ Gaussian $+\bar Y$ & 1.542\,$\pm$\,0.008 & 1.215\,$\pm$\,0.006 & 1.168\,$\pm$\,0.004 & 1.192\,$\pm$\,0.004 & 1.111\,$\pm$\,0.005 & 1.138\,$\pm$\,0.003 & 0.65 & 0.75 \\
L3 q & real & 1.426\,$\pm$\,0.005 & 1.167\,$\pm$\,0.004 & 1.162\,$\pm$\,0.004 & 1.190\,$\pm$\,0.003 & 1.099\,$\pm$\,0.004 & 1.134\,$\pm$\,0.003 & 0.55 & 0.68 \\
\midrule
L3 v & iso. Gaussian & 1.506\,$\pm$\,0.002 & 1.196\,$\pm$\,0.001 & 1.120\,$\pm$\,0.001 & 1.120\,$\pm$\,0.001 & 1.081\,$\pm$\,0.001 & 1.084\,$\pm$\,0.001 & 0.76 & 0.83 \\
L3 v & iso. Gaussian $+\bar Y$ & 1.500\,$\pm$\,0.001 & 1.189\,$\pm$\,0.001 & 1.091\,$\pm$\,0.001 & 1.107\,$\pm$\,0.001 & 1.064\,$\pm$\,0.001 & 1.074\,$\pm$\,0.001 & 0.79 & 0.85 \\
L3 v & cov.\ Gaussian & 1.524\,$\pm$\,0.013 & 1.197\,$\pm$\,0.010 & 1.119\,$\pm$\,0.007 & 1.120\,$\pm$\,0.007 & 1.074\,$\pm$\,0.007 & 1.079\,$\pm$\,0.007 & 0.77 & 0.85 \\
L3 v & cov.\ Gaussian $+\bar Y$ & 1.536\,$\pm$\,0.014 & 1.213\,$\pm$\,0.011 & 1.176\,$\pm$\,0.011 & 1.201\,$\pm$\,0.011 & 1.114\,$\pm$\,0.010 & 1.147\,$\pm$\,0.010 & 0.63 & 0.73 \\
L3 v & real & 1.445\,$\pm$\,0.017 & 1.177\,$\pm$\,0.012 & 1.165\,$\pm$\,0.011 & 1.192\,$\pm$\,0.011 & 1.102\,$\pm$\,0.011 & 1.139\,$\pm$\,0.011 & 0.57 & 0.69 \\
\midrule
L3 o & iso. Gaussian & 1.506\,$\pm$\,0.002 & 1.196\,$\pm$\,0.001 & 1.120\,$\pm$\,0.001 & 1.120\,$\pm$\,0.001 & 1.079\,$\pm$\,0.001 & 1.084\,$\pm$\,0.001 & 0.76 & 0.83 \\
L3 o & iso. Gaussian $+\bar Y$ & 1.511\,$\pm$\,0.002 & 1.201\,$\pm$\,0.001 & 1.115\,$\pm$\,0.001 & 1.108\,$\pm$\,0.001 & 1.083\,$\pm$\,0.001 & 1.078\,$\pm$\,0.001 & 0.79 & 0.85 \\
L3 o & cov.\ Gaussian & 1.526\,$\pm$\,0.005 & 1.200\,$\pm$\,0.003 & 1.118\,$\pm$\,0.003 & 1.118\,$\pm$\,0.004 & 1.077\,$\pm$\,0.003 & 1.080\,$\pm$\,0.004 & 0.78 & 0.85 \\
L3 o & cov.\ Gaussian $+\bar Y$ & 1.531\,$\pm$\,0.005 & 1.211\,$\pm$\,0.003 & 1.167\,$\pm$\,0.004 & 1.191\,$\pm$\,0.004 & 1.109\,$\pm$\,0.003 & 1.136\,$\pm$\,0.004 & 0.64 & 0.74 \\
L3 o & real & 1.421\,$\pm$\,0.005 & 1.166\,$\pm$\,0.004 & 1.160\,$\pm$\,0.003 & 1.188\,$\pm$\,0.003 & 1.098\,$\pm$\,0.003 & 1.133\,$\pm$\,0.004 & 0.55 & 0.68 \\
\midrule
L3 fc & iso. Gaussian & 1.090\,$\pm$\,0.000 & 1.037\,$\pm$\,0.000 & 1.026\,$\pm$\,0.000 & 1.026\,$\pm$\,0.000 & 1.017\,$\pm$\,0.000 & 1.019\,$\pm$\,0.000 & 0.71 & 0.79 \\
L3 fc & iso. Gaussian $+\bar Y$ & 1.091\,$\pm$\,0.000 & 1.037\,$\pm$\,0.000 & 1.037\,$\pm$\,0.000 & 1.051\,$\pm$\,0.000 & 1.021\,$\pm$\,0.000 & 1.038\,$\pm$\,0.000 & 0.44 & 0.58 \\
L3 fc & cov.\ Gaussian & 1.141\,$\pm$\,0.002 & 1.060\,$\pm$\,0.001 & 1.045\,$\pm$\,0.002 & 1.045\,$\pm$\,0.002 & 1.030\,$\pm$\,0.002 & 1.033\,$\pm$\,0.002 & 0.68 & 0.77 \\
L3 fc & cov.\ Gaussian $+\bar Y$ & 1.141\,$\pm$\,0.001 & 1.062\,$\pm$\,0.001 & 1.058\,$\pm$\,0.001 & 1.058\,$\pm$\,0.001 & 1.040\,$\pm$\,0.001 & 1.047\,$\pm$\,0.001 & 0.59 & 0.66 \\
L3 fc & real & 1.076\,$\pm$\,0.003 & 1.032\,$\pm$\,0.002 & 1.054\,$\pm$\,0.003 & 1.057\,$\pm$\,0.003 & 1.033\,$\pm$\,0.003 & 1.045\,$\pm$\,0.003 & 0.25 & 0.41 \\
\midrule
L3 proj & iso. Gaussian & 1.090\,$\pm$\,0.000 & 1.037\,$\pm$\,0.000 & 1.026\,$\pm$\,0.000 & 1.026\,$\pm$\,0.000 & 1.017\,$\pm$\,0.000 & 1.019\,$\pm$\,0.000 & 0.71 & 0.79 \\
L3 proj & iso. Gaussian $+\bar Y$ & 1.090\,$\pm$\,0.000 & 1.035\,$\pm$\,0.000 & 1.042\,$\pm$\,0.000 & 1.047\,$\pm$\,0.000 & 1.023\,$\pm$\,0.000 & 1.038\,$\pm$\,0.000 & 0.48 & 0.57 \\
L3 proj & cov.\ Gaussian & 1.170\,$\pm$\,0.003 & 1.067\,$\pm$\,0.003 & 1.048\,$\pm$\,0.003 & 1.050\,$\pm$\,0.003 & 1.030\,$\pm$\,0.003 & 1.035\,$\pm$\,0.003 & 0.71 & 0.80 \\
L3 proj & cov.\ Gaussian $+\bar Y$ & 1.175\,$\pm$\,0.003 & 1.077\,$\pm$\,0.003 & 1.069\,$\pm$\,0.003 & 1.068\,$\pm$\,0.003 & 1.049\,$\pm$\,0.003 & 1.055\,$\pm$\,0.003 & 0.61 & 0.69 \\
L3 proj & real & 1.121\,$\pm$\,0.004 & 1.051\,$\pm$\,0.004 & 1.063\,$\pm$\,0.004 & 1.065\,$\pm$\,0.003 & 1.040\,$\pm$\,0.004 & 1.052\,$\pm$\,0.003 & 0.46 & 0.58 \\
\midrule
L0 fc & iso. Gaussian & 1.091\,$\pm$\,0.000 & 1.037\,$\pm$\,0.000 & 1.026\,$\pm$\,0.000 & 1.026\,$\pm$\,0.000 & 1.017\,$\pm$\,0.000 & 1.019\,$\pm$\,0.000 & 0.71 & 0.79 \\
L0 fc & iso. Gaussian $+\bar Y$ & 1.089\,$\pm$\,0.000 & 1.036\,$\pm$\,0.000 & 1.040\,$\pm$\,0.000 & 1.041\,$\pm$\,0.000 & 1.022\,$\pm$\,0.000 & 1.034\,$\pm$\,0.000 & 0.54 & 0.62 \\
L0 fc & cov.\ Gaussian & 1.149\,$\pm$\,0.004 & 1.063\,$\pm$\,0.004 & 1.048\,$\pm$\,0.003 & 1.049\,$\pm$\,0.003 & 1.031\,$\pm$\,0.003 & 1.036\,$\pm$\,0.003 & 0.67 & 0.76 \\
L0 fc & cov.\ Gaussian $+\bar Y$ & 1.150\,$\pm$\,0.004 & 1.067\,$\pm$\,0.003 & 1.062\,$\pm$\,0.003 & 1.060\,$\pm$\,0.003 & 1.044\,$\pm$\,0.003 & 1.050\,$\pm$\,0.003 & 0.60 & 0.67 \\
L0 fc & real & 1.105\,$\pm$\,0.005 & 1.045\,$\pm$\,0.004 & 1.058\,$\pm$\,0.003 & 1.058\,$\pm$\,0.003 & 1.037\,$\pm$\,0.004 & 1.047\,$\pm$\,0.003 & 0.45 & 0.55 \\
\bottomrule\end{tabular}}
\caption{$\mathcal R^{\rm lin}_{256}$ at the step-400 checkpoint, effective batch 64, NS5 in float32, for five noise sources and six filters. The last two columns give the reduction fraction $(\mathcal R_0-\mathcal R_\beta)/(\mathcal R_0-1)$ for EMA $0.95$ and Nesterov $0.95$. Intervals are 95\% jackknife intervals over 16 chain groups.}
\label{tab:network-ratios}
\end{table}

\clearpage
\section*{Use of AI tools}
AI tools were used in editing and proofreading the paper, in checking the
proofs, and in implementing the experiment code. The author is responsible for the content of the paper.
\end{document}